\documentclass{article}

\PassOptionsToPackage{numbers, compress}{natbib}

\usepackage[main, final]{neurips_2026}

\usepackage[utf8]{inputenc} 
\usepackage[T1]{fontenc}    
\usepackage{hyperref}       
\usepackage{url}            
\usepackage{booktabs}       
\usepackage{amsfonts}       
\usepackage{nicefrac}       
\usepackage{microtype}      
\usepackage[table,xcdraw]{xcolor}         

\usepackage[pdftex]{graphicx} 

\usepackage{algorithm}
\usepackage{algorithmic}

\usepackage{amsmath}
\usepackage{amssymb}
\usepackage{mathtools}
\usepackage{amsthm}
\usepackage[capitalize,noabbrev]{cleveref}
\usepackage{multirow}
\usepackage{makecell}
\usepackage{subcaption} 
\usepackage{wrapfig}

\theoremstyle{plain}
\newtheorem{theorem}{Theorem}[section]
\newtheorem{proposition}[theorem]{Proposition}
\newtheorem{lemma}[theorem]{Lemma}
\newtheorem{corollary}[theorem]{Corollary}
\theoremstyle{definition}
\newtheorem{definition}[theorem]{Definition}
\newtheorem{assumption}[theorem]{Assumption}
\theoremstyle{remark}
\newtheorem{remark}[theorem]{Remark}

\newcommand{\R}{\mathbb{R}}
\newcommand{\one}{\mathbf{1}}
\newcommand{\relu}{\operatorname{ReLU}}
\newcommand{\norm}[1]{\left\lVert#1\right\rVert}

\newcommand{\pos}[1]{\left(#1\right)_{+}}
\newcommand{\vpos}[1]{\left(#1\right)_{+}}

\newcommand{\res}{e}

\newcommand{\github}{\raisebox{-1pt}{\includegraphics[height=0.9em]{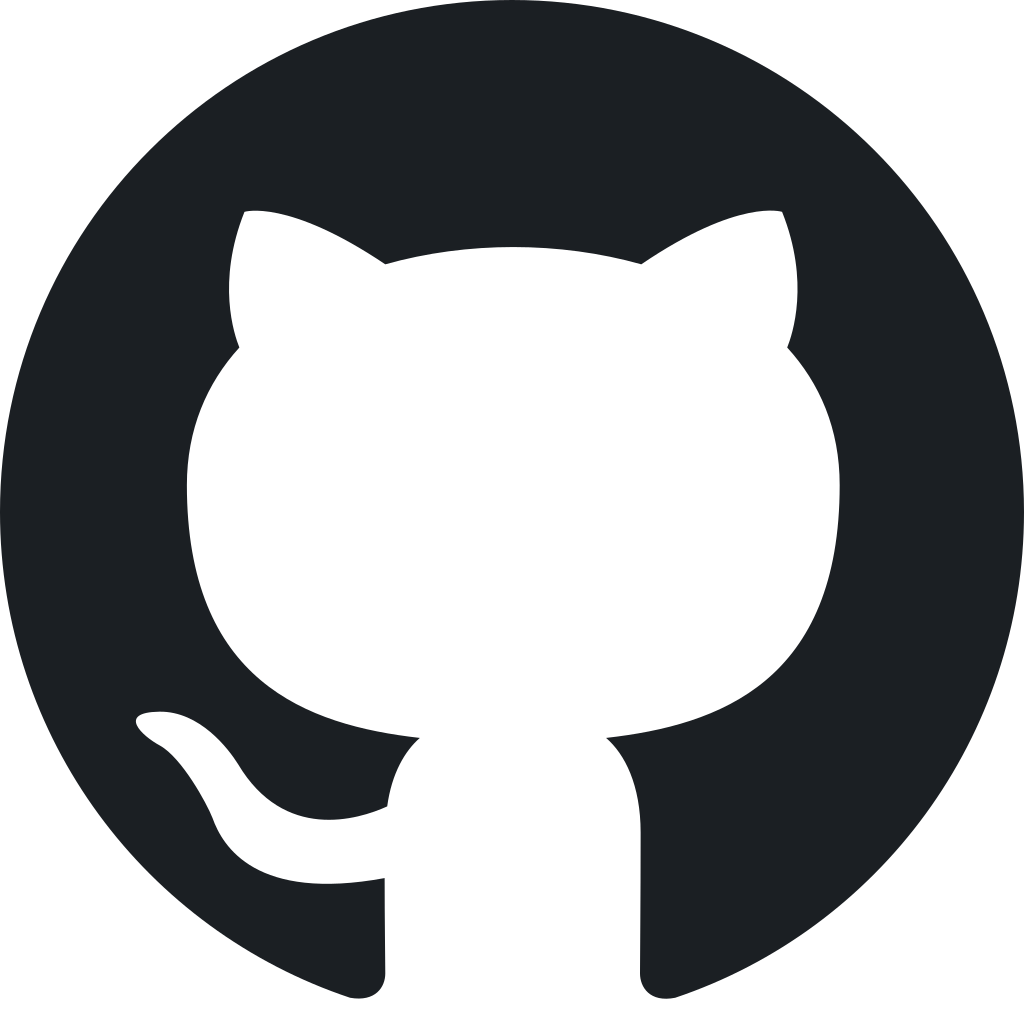}}}

\title{SVDP: Training-Free Contextual Sparsity Predictors \\ for Fast LLM Inference}

\author{%
  \textbf{Georgii~Serbin$^{\spadesuit}$} \quad
  \textbf{Kirill~Koshkin$^{\diamondsuit}$} \quad
  \textbf{Zhongao~Sun$^{\spadesuit}$} \vspace{0.1em} \\
  \textbf{Anastasiya~Bistrigova$^{\diamondsuit}$} \quad
  \textbf{C.C.~Korikov$^{\spadesuit}$} \vspace{0.4em} \\
	\textnormal{$^{\spadesuit}$Huawei Technologies Ltd. \quad $^{\diamondsuit}$ Moscow State University}
	\vspace{0.4em} \\
	\texttt{constantine.korikov@gmail.com}
	\vspace{0.4em} \\
	\github\hspace{0.4em}\textbf{Code}: \url{https://github.com/ckorikov/noah-research/tree/SVDP/SVDP}
}

\hypersetup{pdfauthor={}, pdftitle={}, pdfsubject={}, pdfkeywords={}, pdfproducer={}, pdfcreator={}}
\begin{document}

\maketitle

\begin{abstract}
Contextual sparsity is one of the approaches used to reduce computational complexity in the inference process of large language models (LLMs).
Existing techniques for efficient LLM inference acceleration based on contextual sparsity with minimal accuracy degradation require training sparse pattern predictors.
This paper presents a framework for accelerating the inference of ReGLU-based feed-forward networks (FFNs) within LLMs.
The proposed framework provides a fast, training-free method for building sparse pattern predictors using truncation-aware singular value decomposition (SVD) of the gate projection matrix, along with a threshold calibration algorithm, and inference executors supporting conditional computation on CUDA and CANN devices.
Experiments on three sparse LLMs with an average activation sparsity level of 90\% in the FFNs demonstrate up to a 1.8$\times$ end-to-end decoding time speedup with less than 1\% degradation in benchmark scores on tasks involving math and code generation.
This work advances the deployment of LLMs on edge devices. 
\end{abstract}

\section{Introduction}
\label{sec:introduction}
The remarkable capabilities of LLMs have led to their widespread integration across diverse domains \cite{10.5555/3495724.3495883, 10748304}.
While on-device AI deployment offers advantages in privacy, latency, and cost compared to server-based solutions \cite{10730112}, the substantial computational requirements of LLM inference present significant challenges for edge deployment.

A key characteristic of ReGLU-based FFNs within LLMs is dynamic (or contextual) sparsity -- the phenomenon where only a small subset of neurons are activated for any given input \cite{li2023the}.
This inherent property enables inference acceleration by reducing both computational overhead and memory bandwidth requirements during generation \cite{liu2023deja}.

A prevalent approach for leveraging contextual sparsity involves employing lightweight, learnable sparsity predictors \cite{liu2023deja}.
These predictors generate input-specific sparse activation patterns prior to FFN forward passes.
The sparse pattern identifies neuron indices that will produce positive dot products with the input vector. In ReGLU-based FFNs, computations for the remaining neurons can be safely skipped without affecting output correctness, as the activation function nullifies negative values.
Nevertheless, prediction errors may arise from false negative classifications, which compromise model accuracy.
Effective predictors must balance high recall rates with substantial sparsity ratios while maintaining sufficient speed to reduce end-to-end (E2E) text generation latency.

In this study, we present a simple yet effective technique leveraging low-rank singular value decomposition (SVD) to build sparse pattern predictors, thereby enabling significant inference acceleration in ReGLU-based LLMs.
Our methodology employs low-rank approximations of gate projection weight matrices in each layer as lightweight predictors of FFN sparsity during generation. Our method provides theoretical guarantees on prediction error.

To enhance the accuracy of SVD-based predictors, we employ a combination of techniques:
\begin{itemize}
    \item \textbf{Calibrated bias}. Direct application of SVD approximations as predictors results in notable accuracy degradation.
	However, we observe that these predictors effectively separate active and inactive subsets of neurons.
	Introducing neuron-specific thresholds -- implemented as bias terms -- effectively mitigates this issue.
	These biases are calibrated offline using a custom greedy algorithm that does not require gradient computation.
    \item \textbf{Data whitening technique}. Given the predictor speed constraints, the rank must be significantly smaller than the hidden dimension.
	To improve prediction accuracy under these conditions, we apply truncation-aware data whitening \cite{wang2025svdllm_v2, wang2025svdllm_v1, yuan2025asvdactivationawaresingularvalue}.
	This technique manages activation outliers by incorporating input activation statistics, requiring only a small calibration dataset.
\end{itemize}

Implementation-level optimizations further enhance inference speedup:
\begin{itemize}
    \item \textbf{Sequential computation of gate and up projections}. Gate and up projections are computed sequentially, enabling further acceleration through early elimination of false positive predictions.
	After computing gate projections, falsely activated neurons are identified and excluded from subsequent up and down projections, thereby increasing sparsity levels in the final FFN stages.
    \item \textbf{Inference executors}. 
	Our implementations support conditional computation on CUDA \cite{cuda-c-programming-guide} and CANN \cite{cann} devices.
	For CUDA, we employ custom kernels rather than Triton-based abstractions \cite{10.1145/3315508.3329973}, prioritizing enhanced low-level control.
	On Huawei Ascend hardware \cite{chen2019davinci}, comparable performance gains are achieved exclusively through high-level PyTorch APIs, demonstrating effective utilization of the platform's native software stack.
\end{itemize}

\cref{fig:framework-overview} illustrates our framework architecture.
\begin{figure}
    \centering
    \includegraphics[width=0.98\textwidth]{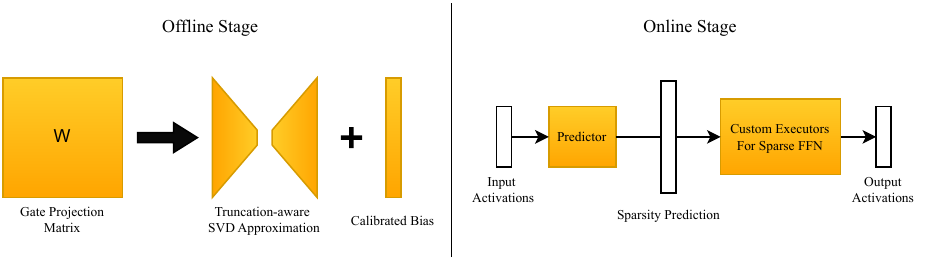}
    \caption{Overview of our framework.
	During the offline stage, sparsity predictors are constructed using truncation-aware SVD approximation with calibrated per-neuron thresholds.
	During the online stage, SVD-based predictors precede each FFN to exploit contextual sparsity and accelerate generation.}
    \label{fig:framework-overview}
\end{figure}
Our key contributions are summarized as follows:
\begin{enumerate}
    \item We propose training-free SVD-based contextual sparsity predictors for ReGLU-based LLM inference acceleration;
    \item We introduce calibrated bias terms that significantly improve prediction accuracy, with biases tuned on small datasets using a custom gradient-free algorithm;
    \item We develop executors supporting sparse FFN inference on CUDA and CANN devices;
    \item We achieve up to 1.8$\times$ end-to-end inference speedup with minimal accuracy degradation.
\end{enumerate}

\section{Related work}
\label{sec:related-work}

Sparsification is one of the most effective paradigms for accelerating LLM inference, typically categorized into \textbf{static} and \textbf{dynamic sparsity} methods. \textbf{Static sparsity} removes redundant parameters before deployment~\cite{ICLR2024_316648eb}, reducing computation and memory costs but often sacrificing accuracy at high sparsity levels.
In contrast, \textbf{dynamic sparsity} -- also known as contextual sparsity -- adapts sparsity patterns at runtime based on input context, enabling flexible efficiency-accuracy trade-offs.

\textbf{Predictor-based dynamic sparsity} methods, such as Deja Vu \cite{liu2023deja} and ShadowLLM \cite{akhauri-etal-2024-shadowllm}, learn lightweight neural predictors to estimate layer-wise sparsity. \citet{shrestha2025polar} focuses on batched inference, delivering up to 2.2$\times$ E2E speedups by exploiting head sparsity in the Attention. In contrast, our work focuses on single-query execution, which is common for edge devices and where E2E latency is dominated by FFN.
\textbf{Training-free dynamic sparsity} eliminates learned predictors and instead derives sparsity from activation statistics or calibrated thresholds.
\citet{dong2024promptprompted} selects active ``expert’’ neurons using prefill activations, while \citet{ma2024activationsmattertrainingfreemethods} introduces offline-calibrated thresholds allowing adaptive activation per layer.
Both maintain fixed sparsity patterns within a prompt.
\citet{yang-etal-2026-sparse} updates sparsity during generation but requires recomputation of previous states, hindering KV-cache reuse.
\citet{liu2024trainingfreeactivationsparsitylarge} and \citet{chen2026wina} exploit input sparsity, which differs from activation sparsity in that it arises from redundancy in the input representations rather than neuron activations, using calibrated thresholds.
\citet{sparseinfer2025} proposes a calibrated solution for ReLU-fied LLMs using sign-bit comparisons, which lacks theoretical justifications.

Dynamic sparsity can also underpin efficient on-device inference frameworks.
\citet{song2024powerinferfastlargelanguage} accelerates LLMs on consumer GPUs by preloading frequently activated neurons into GPU memory, minimizing CPU transfers.
PowerInfer achieves up to 11.69$\times$ speedup, particularly in CPU/GPU hybrid settings, whereas our work targets homogeneous GPU inference.

The effectiveness of dynamic sparsity depends heavily on activation functions.
ReLU activations naturally induce sparse activations, unlike smoother functions such as Swish or SiLU used in modern LLMs \cite{Zhang2024ReLU2WD}.
Recent studies show that replacing them with ReLU variants plus fine-tuning recovers substantial sparsity \cite{mirzadeh2023relustrikesbackexploiting}.
\citet{song2024prosparse} achieved $\sim$90\% sparsity in LLaMA2-7B through progressive regularization and activation threshold shifting, while \citet{song2024turbosparseachievingllm} used double ReLU (dReLU) activation function to sparsify models like Qwen2-7B and Mistral-7B.

Beyond pruning, \textbf{SVD} offers complementary compression capabilities.
\citet{yuan2025asvdactivationawaresingularvalue} leverages activation-aware transformations to attain 10--30\% model compression, and \citet{wang2025svdllm_v1} provides theoretical connections between discarded components and the compression loss.

\citet{khaki2025sparselora} integrates SVD-based estimators into parameter-efficient fine-tuning, dynamically slicing weights by input relevance.
However, sparse computations are applied only to \textit{context tokens} (the prefix tokens provided as input), while \textit{output tokens} (the target tokens used for loss computation) remain densely processed to avoid performance degradation.
Achieving accelerated and high-quality decoding of LLMs through contextual sparsity remains an open challenge not addressed by SparseLoRA.

\section{Preliminaries}

Consider a ReGLU-activated FFN during decoding, defined as:
\begin{equation}
\text{FFN}(x) = W^\mathrm{down}\left(\text{ReLU}(W^\mathrm{gate}x) \odot \left(W^\mathrm{up}x\right)\right),
\end{equation}
where $x \in \mathbb{R}^d$ is the input activation vector, and $W^\mathrm{gate}, W^\mathrm{up} \in \mathbb{R}^{D \times d}$, and $W^\mathrm{down} \in \mathbb{R}^{d \times D}$ are the gate, up-projection, and down-projection matrices, respectively.
Here, $d$ denotes the hidden size and $D$ the intermediate size.
Let $X = [x_1,\,\dots,\, x_N] \in \mathbb{R}^{d \times N}$ denote the hidden states of $N$ calibration tokens, and let $r$ denote the predictor rank.
Finally, let $n(x) \in \mathbb{R}^D$ denote the neuron values
\begin{equation}
	n(x) = \text{ReLU}(W^\mathrm{gate}x) \odot \left(W^\mathrm{up}x\right).
\end{equation}
For brevity, we omit layer indices.

Our objective is to obtain a predictor $P$ that outputs a binary mask $m = P(x) \in \{0,1\}^D$ over the $D$ intermediate neurons. 
This mask identifies active neurons by selecting the corresponding rows of $W^\mathrm{gate}$ and $W^\mathrm{up}$, together with their aligned columns of $W^\mathrm{down}$, yielding sparse computation algebraically equivalent to the dense FFN due to ReLU sparsity.

\section{SVD-based sparsity predictors}
\label{sec:svd-predictors}

This section introduces \emph{training-free SVD-based sparsity predictors} (SVDP), designed for efficient inference in ReGLU-based LLMs.
Our approach consists of two stages: an offline phase for predictor construction and an online phase for accelerated sparse inference. 

\subsection{Training-free predictor building}
The predictors are constructed by approximating the gate projection matrix $W^\mathrm{gate}$ using a low-rank SVD.
This approach offers a theoretical upper bound on local error (see~\cref{sec:math-proofs-upper-error-bound}).

We propose predictors of the affine form
\begin{equation} 
P(x) = H(A B x + b),
\end{equation}
where $H: \mathbb{R} \to \{0,1\}$ is the Heaviside step function applied element-wise.
$H(z) = 1$ if $z > 0$ and $0$ otherwise. The product $AB$ approximates $W^\mathrm{gate}$, while $b \in \mathbb{R}^D$ adjusts per-neuron decision thresholds.

\subsubsection{Data-aware SVD}
\label{subsec:svd-based-decomposition}
To improve the accuracy of the SVD approximation, we adopt a data whitening technique (\cref{alg:svd-data-whitening}).
This minimizes the Frobenius norm error $\|WX - ABX\|_F$ on calibration data $X$, by incorporating input statistics through a whitening matrix $S$:
\begin{equation}
	W = (WS) S^{-1} = (U\Sigma V^\top) S^{-1} \approx U_r \Sigma_r V_r^\top S^{-1} = AB.
\end{equation}	
We follow~\citet{wang2025svdllm_v1} in choosing $S$ as the Cholesky decomposition of $X X^\top$, which outperforms ASVD in terms of perplexity at low ranks, as shown in Table~2 of \citet{yuan2025asvdactivationawaresingularvalue}.
The predictor construction pipeline is performed independently for each layer.
Throughout, we fix the predictor rank $r$ across layers (see~\cref{sec:rank-size} for a discussion on rank selection). 

\subsubsection{Bias calibration}
\label{subsec:threshold-calibration}

Empirically we observe that the SVD-based sparsity predictors exhibit good separation between active and inactive neurons; however, naive thresholding (e.g., at zero) leads to substantial accuracy degradation.
This effect arises from a shift in the predictor output distribution induced by truncation of singular value components, as illustrated in \cref{fig:distribution-shift}.
\begin{figure}[!h]
	\centering
	\includegraphics[width=0.55\textwidth]{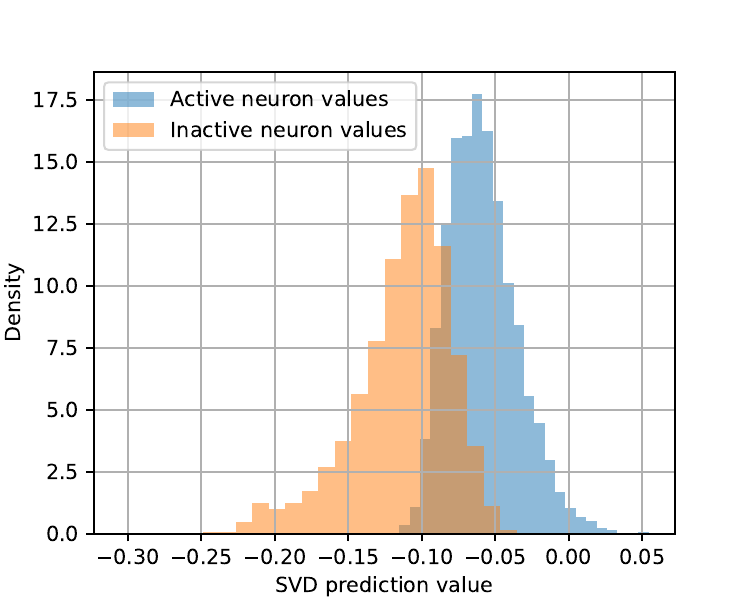}
	\caption{Distribution of SVD-based predictor outputs ($r=256$) for a sample neuron obtained from approximately 25,000 tokens.
	Although active and inactive states are well separated, zero thresholding causes significant false negatives for active states.}
	\label{fig:distribution-shift}
\end{figure}

To compensate for this distribution shift, we introduce a bias term in the second layer of the SVD-based predictor.
The central problem is to calibrate this bias so as to achieve a prescribed sparsity level while limiting the associated accuracy loss.
Simple schemes that tune thresholds to match a target sparsity or recall on a calibration set were empirically found to yield suboptimal trade-offs.

Instead, we employ a fast greedy calibration algorithm, which explicitly accounts for the non-uniform impact of different neurons on model accuracy.
More precisely, from a computational standpoint, all neurons are equally valuable for acceleration, since pruning any neuron reduces the cost by the same amount, whereas their contributions to output accuracy differ substantially.
Consequently, some neurons require high recall, whereas others can tolerate higher false-negative rates without noticeably affecting overall performance.

Formally, we minimize a cost function $C:\, \mathbb{R}^{D} \rightarrow \mathbb{R}$
\begin{equation}
C(b)=\sum_{i=1}^D \alpha_i\left[\sum_{t=1}^N  \left(1 - P\left(x_t\right)\right) \cdot n^2\left(x_t\right) \right]_i
\end{equation}
over bias $b$ on calibration data $X$ with a sparsity constraint
\begin{equation}
\frac{1}{DN}\sum_{i=1}^D\sum_{t=1}^N \left[1 - P\left(x_t\right)\right]_i \geq s,
\end{equation}
where
\begin{itemize}
\item $s \in (0,1]$ --- the desired sparsity ratio hyperparameter that controls the speed-accuracy trade-off;
\item $\alpha_i = \left\| W^\mathrm{down}_{:, i} \right\|_2^2 $ --- the weight coefficient that determines the importance of $i$-th neuron.
\end{itemize}
The cost function aggregates values of neurons falsely predicted to be inactive over the calibration data, reflecting FFN local error.

We solve the problem with a greedy algorithm, which can be briefly described as follows:
\begin{enumerate}
\item Initialize the thresholds ($\tau=-b$) in such a way that the prediction ensures high recall (likely due to low predicted sparsity ratio).
\item Iteratively increase the predicted sparsity ratio via greedy threshold updates. 

At each iteration, the threshold is increased for the neuron, for which the growth of the cost function is minimal.
\item Stop when the predicted sparsity on the calibration data reaches $s$. 
\end{enumerate}
The details are provided in \cref{alg:bias-calibration}.
Under certain conditions, our greedy algorithm is guaranteed to provide an optimal solution (see~\cref{sec:math-proofs-greedy-threshold}).

To properly handle significantly different sparsity ratios across layers, we set $s$ layer-wise
\begin{equation}
	s^{(l)}=s \cdot \hat{s}^{(l)},\quad l=0,\,1,\ldots,
\end{equation}
where $l$ is the layer index and $\hat{s}^{(l)}$ is the estimated sparsity on calibration data $X$.

Trainable predictors require collecting hidden activations of hundreds of thousands of token samples \cite{shrestha2025polar,song2024prosparse} for supervised learning. In contrast, the proposed predictor construction is fast and data-efficient. Empirically we observe that only a few thousand calibration tokens are sufficient for 7B models (see \cref{sec:sensitivity-analysis}).

Overall, to construct SVDP for each layer we collect the calibration hidden states, obtain $A$ and $B$ from the data-whitened SVD of \cref{alg:svd-data-whitening}, and calibrate the bias with \cref{alg:bias-calibration}. 

\subsection{Runtime execution}
\label{sec:online_phase}

Only the FFNs are modified, while the rest of the model remains untouched, enabling seamless integration with existing attention optimizations.
During inference, the prefill stage remains unchanged.
At the decoding stage, we leverage CUDA and CANN executors for sparse FFN computation.
Our predictor comprises two linear layers, requiring no specialized implementation.

\paragraph{Sparsity prediction pipeline}
To partially hide the latency of the predictor, methods, such as Deja Vu~\cite{liu2023deja} and Polar Sparsity~\cite{shrestha2025polar}, overlap predictor execution with the Attention layer via asynchronous computation. However, such pipelining may be infeasible on certain hardware platforms (see \cref{sec:inference-executors-details}). 
Given our focus on local inference for edge deployment scenarios, similar to PowerInfer engine, we do not employ this technique.
Instead, our predictors are inserted at the beginning of each FFN and sequentially generate sparsity patterns for their respective gate projections.

\paragraph{Gate/up execution order}
A prevalent optimization for efficient deployment of LLMs with gated FFNs involves parallelized execution of the gate and up projections. 
While effective in dense computation scenarios, this approach may not be optimal for local inference systems leveraging sparse activations.
We delineate two distinct computational methodologies to address this:
\begin{itemize}
	\item \textit{Parallel pipeline} e.g., PowerInfer. 
	A predictor first generates a sparsity mask, which is concurrently applied to independently compute both gate and up projections. 
	This maintains identical sparsity levels across both projections during execution.
	\item \textit{Sequential pipeline} e.g., SparseInfer and our approach. 
	The gate and up projections are computed sequentially, eliminating false positives: 
	neurons predicted as active are revalidated after gate projection computation, further increasing sparsity in subsequent projections.
\end{itemize}

\paragraph{Algorithmic complexity}
In our case, corrected sparsity equals or exceeds the predicted sparsity $s$ by filtering erroneous predictions.
Thanks to the sequential pipeline, we can use more lightweight predictors that provide relatively low sparsity at the gate projection stage without compromising the FFN's performance.
A dense FFN requires $3dD$ multiplications (excluding element-wise product), whereas our sparse approach reduces this to:
\begin{equation}
    \underbrace{r(d + D)}_{\text{predictor}} + \underbrace{d(1-s)D}_\mathrm{gate} + \underbrace{2d(1-s')D}_{\text{up/down}},
\end{equation}
where $s$ is the predicted sparsity and $s'$ the corrected sparsity.
For example, in ProSparse-LLaMA2-7B \cite{song2024prosparse} ($d\!=\!4096$, $D\!=\!11008$, $s\!=\!0.5$, $s'\!=\!0.9$, $r\!=\!256$), operations decrease by 3.8$\times$.
Note that hardware-level dynamic sparsity challenges prevent this from directly translating to wall-clock speedup.
See \cref{sec:detailed-profiling} for sparsity broken down by source and the decode latency breakdown.

\paragraph{Hardware platforms}
CUDA and CANN devices are supported for sparse FFN inference. See implementation details in \cref{sec:inference-executors-details}. 

\section{Experimental results}
\label{sec:experimental-results}
\subsection{Experimental setup}
\paragraph{Hardware} A single consumer-grade NVIDIA RTX 3090 GPU with memory size of 24 GB
is used for both quality and performance evaluation. 
The inference is performed entirely on GPU.
For evaluation on NPU, a single Huawei Ascend 310P3 \cite{ascend-310p3} device is used. The process of constructing predictors takes approximately 10 minutes for 7B models.

\paragraph{Software} We use the vLLM framework \cite{kwon2023efficient}, preferring it to the Transformers framework \cite{wolf-etal-2020-transformers} for the experiments, as it better reflects real-world deployment scenarios.
Notably, we do not utilize any server-specific features of vLLM, such as tensor parallelism.
Inference is performed in half precision, whereas predictor construction in the offline stage uses double precision.
Calibration data and the whitening/bias split are described in \cref{sec:bias-calibration}.

\paragraph{Models}
We evaluate our solution across three open-source sparse 7B language models:
\textbf{ProSparse-LLaMA2-7B}~\cite{song2024prosparse} and two instruction-tuned models -- \textbf{SparseQwen2-7B} and \textbf{TurboSparse-Mistral-Instruct}~\cite{song2024turbosparseachievingllm}.
The instruction-tuned models incorporate double ReLU (dReLU) activation in their FFN layers:
\begin{equation}
\text{FFN}(x)\!=\!W^\mathrm{down}\!\left(\text{ReLU}\!\left(W^\mathrm{gate}x\right)\!\odot\!\text{ReLU}\!\left(W^\mathrm{up}x\right)\right).
\end{equation}

\paragraph{Quality evaluation}
We evaluate model performance using a diverse set of open-source benchmarks spanning reasoning, coding, knowledge, and multilingual understanding: 
\textbf{GSM8K}~\cite{cobbe2021trainingverifierssolvemath} (math reasoning), 
\textbf{HumanEval}~\cite{chen2021evaluatinglargelanguagemodels} (code generation), 
\textbf{ARC-E/ARC-C}~\cite{clark2018thinksolvedquestionanswering} (scientific and commonsense reasoning), 
\textbf{TriviaQA}~\cite{joshi2017triviaqalargescaledistantly} (open-domain QA), 
\textbf{MBPP}~\cite{austin2021programsynthesislargelanguage} (small Python tasks), 
\textbf{BBH}~\cite{suzgun2022challengingbigbenchtaskschainofthought} (complex reasoning), 
and \textbf{CMMLU}~\cite{li2024cmmlumeasuringmassivemultitask} (Chinese multi-task understanding).
We adopt the \textbf{UltraEval} framework~\cite{he2024ultraeval}, aligning with prior ProSparse evaluations.
All benchmarks use text generation rather than perplexity scoring to better reflect real-world LLM usage.
For all benchmarks we use greedy sampling.
Versions and licenses for used assets are listed in \cref{sec:licenses}.

\paragraph{Performance evaluation}
Model performance is measured using the \textbf{PyTorch profiler}~\cite{Ansel_PyTorch_2_Faster_2024}.
The results are averaged across multiple samples from different language domains.
E2E latency includes the prefill stage and 200-token generation.
Speedup results are stable across independent runs under identical experimental conditions, with a sample standard deviation of $\sigma < 0.001\times$ computed using the closed-form formula. 
Given this low run-to-run variability, error bars are omitted for brevity.  

\subsection{Performance under different sparsity ratios}
The desired sparsity hyperparameter $s$ from \cref{alg:bias-calibration} influences the calibrated thresholds and can be used to trade off accuracy and acceleration. As \cref{tab:7b-sparsity} shows, lower thresholds lead to lower sparsity and lower acceleration. In contrast, higher thresholds provide more speedup but may degrade quality.
Long generation tasks are the most sensitive, while for short answer tasks the sparsity can be significantly increased. With our approach we can achieve up to $1.8\times$ E2E generation speedup with less than one percentage point decrease in the average score across the evaluated benchmarks. See \cref{sec:detailed-acc-results} for detailed evaluation results.

\setlength{\tabcolsep}{3.5pt}
\begin{table}
\caption{Accuracy when varying hyperparameter $s$ for different models. $s$ implicitly sets sparsity, which scales with end-to-end speedup. The accuracy is averaged across all benchmarks. See \cref{sec:detailed-acc-results} for complete evaluation results.}
\label{tab:7b-sparsity}
\centering
\footnotesize
  \begin{tabular}{ccllllll}
      \toprule
      & & \multicolumn{2}{c}{ProSparse-LLaMA2-7B} & \multicolumn{2}{c}{TurboSparse-Mistral-Instruct} & \multicolumn{2}{c}{SparseQwen2-7B} \\
      \cmidrule(lr){3-4} \cmidrule(lr){5-6} \cmidrule(lr){7-8}
      Method & $s$ & Avg. Score, \% & E2E Speedup & Avg. Score, \% & E2E Speedup & Avg. Score, \% & E2E Speedup \\ \midrule
      Dense & $-$ & 38.92 & 1.00$\times$ & 54.30 & 1.00$\times$ & 65.84 & 1.00$\times$ \\
      \midrule
      \multirow{4}{*}{SVDP}
      & 0.4 & 38.66\,(-0.26) & 1.61$\times$ & 53.72\,(-0.58) & 1.67$\times$ & 65.41\,(-0.43) & 1.78$\times$ \\
      & 0.5 & 38.78\,(-0.14) & 1.64$\times$ & 53.56\,(-0.74) & 1.73$\times$ & 65.20\,(-0.64) & 1.81$\times$ \\
      & 0.6 & 38.77\,(-0.15) & 1.67$\times$ & 53.18\,(-1.12) & 1.79$\times$ & 65.13\,(-0.71) & 1.85$\times$ \\
      & 0.7 & 38.57\,(-0.35) & 1.71$\times$ & 53.44\,(-0.86) & 1.84$\times$ & 64.49\,(-1.35) & 1.86$\times$ \\
      \bottomrule
  \end{tabular}
\end{table}

{
\setlength{\tabcolsep}{1.5pt}
\begin{table*}[t]
\centering
\caption{Comparison with other methods of sparsity prediction on ProSparse-LLaMA2-7B, evaluated within our framework: asynchronous overlap is disabled for all methods, including SVDP.
		For Polar Sparsity, we use threshold-based selection instead of top-k.
		SVDP outperforms PowerInfer’s predictors in most benchmarks. A more lightweight SVDP~($r=256$) enables additional speedups without compromising quality due to sequential pipeline.
		}
\label{tab:other-methods}
\footnotesize
\begin{tabular}{lcccccccccc} 
\toprule
\multirow{2}{*}{\makecell[c]{Predictor}}
& TriviaQA & ARC-E & ARC-C & HumanEval & MBPP & GSM8K & BBH & CMMLU & Avg. & E2E \\
& acc &acc &acc &pass@1 &pass@1 &acc &acc &acc & Score & Speedup \\
\midrule
Dense & 62.94 & 75.42 & 53.58 & 17.68 & 22.59 & 11.45 & 35.85 & 31.84 & 38.92 & 1.00$\times$ \\
\midrule
GRIFFIN, 20\% sparsity & 59.77 & 74.75 & 53.33 & \underline{15.85} & \textbf{22.18} & 8.95 & 35.49 & \underline{32.01} & 37.79 & 1.14$\times$ \\
GRIFFIN, 50\% sparsity & 35.74 & 72.77 & 51.96 & 12.20 & 18.17 & 4.47 & 35.32 & 31.88 & 32.81 & 1.46$\times$ \\
	Deja Vu & 60.78 & 75.21 & 53.50	 & 2.44 & 1.54	 & 7.43 & 35.44 & 30.67 & 33.38 & 1.64$\times$ \\
	Polar Sparsity & \textbf{62.88} & 75.25 & 53.50 & 6.10 & 8.11 & 10.24 & \underline{35.61} & \textbf{32.18} & 35.48 & 1.64$\times$ \\
	PowerInfer & 61.10 & 74.54 & 52.99 & \underline{15.85} & 21.15 & \underline{10.69} & 35.29 & 31.80 & 37.93 & 1.63$\times$ \\
	\rowcolor{blue!10}SVDP~($r=1024$) & 62.66 &	\textbf{75.38} &	\underline{53.58} &	\textbf{17.68} &	\textbf{22.18}&	10.54&	\textbf{35.86} &	31.85&	\underline{38.72}	& 1.61$\times$ \\
	\rowcolor{blue!10}SVDP~($r=256$) & \underline{62.74} & \underline{75.34} & \textbf{53.67} &  \textbf{17.68} &  \underline{22.07}&	\textbf{11.22}&	35.58 &	31.94 &	\textbf{38.78}	& 1.64$\times$ \\
\bottomrule
\end{tabular}
\end{table*}
}

\subsection{Comparison with other methods}
We compare SVDP against training-free sparsity methods (GRIFFIN \cite{dong2024promptprompted}) and learnable predictors (Deja Vu \cite{liu2023deja}, PowerInfer \cite{song2024powerinferfastlargelanguage}, Polar Sparsity \cite{shrestha2025polar}).
GRIFFIN relies on prompt-based statistical sparsity patterns, while others use trained neural networks. 
For Deja Vu and Polar Sparsity, we trained predictors following their official repository instructions (see details in \cref{sec:baseline-configs}). 
For PowerInfer, we used their recommended pretrained predictors. 
For SVDP, we tested configurations with $r=1024$ and $r=256$, where in the latter configuration the predictor yields a lower level of sparsity but operates faster.
In this set of experiments everything runs in one synchronous sequential pipeline, so the tables compare the predictors themselves under identical conditions, not the methods as systems.

\paragraph{Benchmarks}
\cref{tab:other-methods} shows that our SVDP matches PowerInfer's accuracy and performance while requiring less data and compute for construction.
GRIFFIN struggles with stable performance at high sparsity ratios, particularly for long generations.
The sequential pipeline implementation enables reducing SVDP size by 4$\times$ without compromising either accuracy or inference speed.

\paragraph{ROC-AUC analysis}
We also evaluate predictor separation ability via layer-wise ROC-AUC scores on held-out test sequences (\cref{fig:rocauc-comparison}). 
SVDP outperforms learnable baselines on unseen data. Moreover, SVD-based predictors with 4$\times$ smaller size are on par with trainable predictors.
Trained neural networks may suffer from out-of-distribution sensitivity due to their data-dependent nature, while SVD predictors directly leverage the LLM's internal weight structure.

\begin{figure}[htbp]
    \centering
    \begin{subfigure}[b]{0.495\textwidth}
        \centering
        \includegraphics[width=\textwidth]{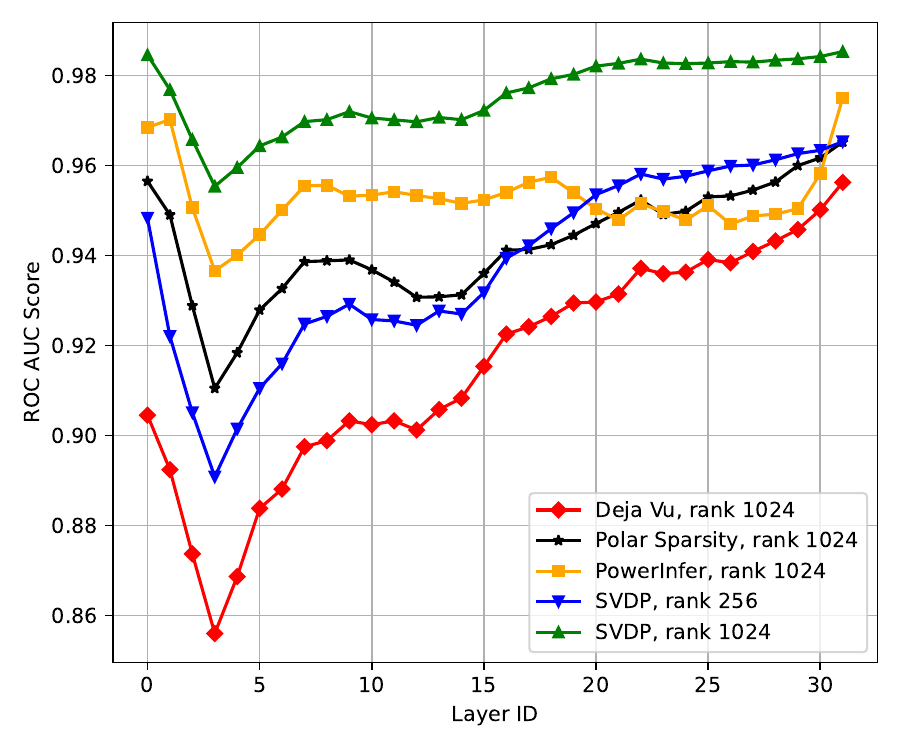}
        \caption{ }
        \label{fig:rocauc-comparison}
    \end{subfigure}
    \hfill
    \begin{subfigure}[b]{0.495\textwidth}
        \centering
        \includegraphics[width=\textwidth]{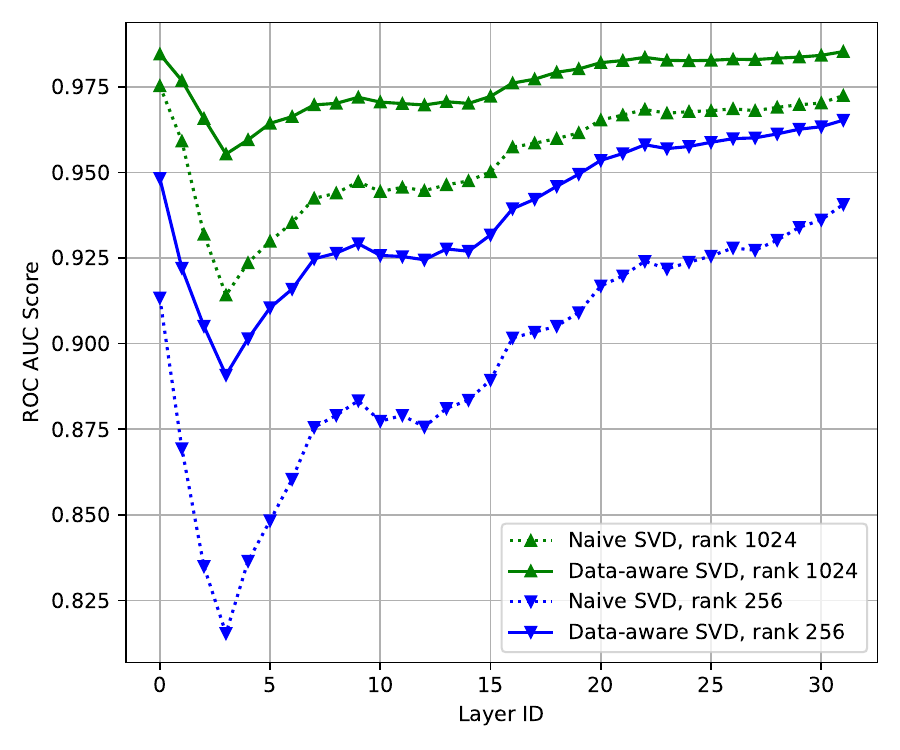}
        \caption{ }
        \label{fig:rocauc-ablation}
    \end{subfigure}
    \caption{Layer-wise ROC-AUC score of predictors averaged across several sequence samples for ProSparse-LLaMA2-7B. Score of 1 indicates perfect separation and AUROC score of 0.5 corresponds to random guessing.
	 (a) Comparison with trained neural networks. SVDP outperforms PowerInfer on all layers.
	 (b) Evaluation of the impact of data whitening on the predictive ability. Data whitening significantly improves the prediction quality, especially at small rank.}
    \label{fig:rocauc}
\end{figure}

\subsection{Ablation study}
{
	\begin{table*}[t]
	\setlength{\tabcolsep}{2pt}
	\footnotesize
	\centering
	\caption{SVD sparsity estimator ablation study. 
	$^\dagger$Averaged across all benchmarks (see Appendix~\ref{sec:detailed-acc-results} for detailed results).
	Data whitening substantially enhances model quality. Bias calibration allows near lossless inference.
	}
	\label{tab:components-breakdown-simple}
	\begin{tabular}{p{2.5cm}ccccccccc}
		\toprule
		\multirow{2}{*}{\makecell[c]{Configuration}}
		& \multicolumn{3}{c}{ProSparse-LLaMA2-7B}
		& \multicolumn{3}{c}{TurboSparse-Mistral-Instruct}
		& \multicolumn{3}{c}{SparseQwen2-7B} \\
		\cmidrule(lr){2-4} \cmidrule(lr){5-7} \cmidrule(l){8-10}
		& GSM8K & HumanEval & $\text{Avg.}^\dagger$ & GSM8K & HumanEval & $\text{Avg.}^\dagger$ & GSM8K & HumanEval & $\text{Avg.}^\dagger$ \\
		\midrule
		Dense baseline & 11.45 & 17.68 & 38.92 & 61.71 & 34.15 & 54.30 & 77.1 & 64.02 & 65.84 \\
		\midrule
		\multicolumn{10}{l}{\textit{Parallel pipeline}} \\
		\hspace{1em}Naive SVD & 6.60 & 10.98 & 34.08 & 60.20 & 24.39 & 51.52 & 74.00 & 53.66 & 62.05 \\
		\hspace{1em}+ Data whitening & 10.16 & 16.46 & 37.10 & 60.88 & 37.20 & 53.02 & 75.89 & 61.59 & 63.99 \\
		\hspace{1em}+ Bias calibration & 10.54 & 17.68 & 38.72 & 62.24 & 36.59 & 54.34 & 76.65 & 62.20 & 65.37 \\
		\midrule
		\multicolumn{10}{l}{\textit{Sequential pipeline}} \\
		\hspace{1em}Naive SVD & 1.90 & 0.61 & 9.69 & 9.17 & 1.83 & 30.87 & 63.00 & 12.20 & 51.60 \\
		\hspace{1em}+ Data whitening & 6.90 & 8.54 & 31.00 & 48.60 & 29.27 & 46.49 & 72.93 & 57.32 & 60.80 \\
		\hspace{1em}+ Bias calibration & 11.22 & 17.68 & 38.78 & 61.41 & 33.54 & 53.56 & 76.80 & 62.80 & 65.22 \\
		\bottomrule
		\end{tabular}
\end{table*}
}

\cref{tab:components-breakdown-simple} shows ablation of SVD sparsity estimators. Starting from SparseLoRA's baseline~\cite{khaki2025sparselora}, referred to as Naive SVD, we sequentially employ techniques:
 (1) \textit{data whitening} (Sec.~\ref{subsec:svd-based-decomposition}), (2) \textit{bias calibration} (Sec.~\ref{subsec:threshold-calibration}).
Bias calibration significantly improves accuracy by accounting for distribution shift after low-rank approximation. 
Data whitening substantially enhances ROC-AUC separation on held-out test sequences (\cref{fig:rocauc-ablation}), independent of bias calibration. The improvement is most noticeable for small rank.
For the ablation study against standard threshold selection algorithms see \cref{sec:alternative-biases}.

\subsection{Inference executors} 
\cref{tab:npu-res} presents FFN performance of our CUDA and CANN executors across varying sparsity ratios.
Notably, we achieve comparable performance gains on the Ascend hardware using only high-level PyTorch APIs, eliminating the need for manual low-level instruction optimization. 
This result facilitates deeper exploration of contextual sparsity optimizations within the DaVinci architecture. 
It is worth noting that when the sparsity value approaches 90\%, the L1 \cite{chen2019davinci} cache hit rate reaches almost 100\%, resulting in a significant speedup increase.
{
\begin{table}
	\caption{FFN performance of our CUDA and CANN executors at different sparsity ratios. 
	The speedups are compared to the corresponding dense implementations.
	Predictor latency is not included. 
	For this experiment, sparsity ratio is fixed across all linear layers within FFN.
	As sparsity increases, computation becomes more efficient. A remarkable acceleration up to $30\times$ can be achieved with our CANN executor.}
	\centering
	{
		\footnotesize
		\begin{tabular}{c>{\centering\arraybackslash}p{1.3cm}>{\centering\arraybackslash}p{2cm}>{\centering\arraybackslash}p{1.3cm}>{\centering\arraybackslash}p{2cm}>{\centering\arraybackslash}p{1.3cm}>{\centering\arraybackslash}p{2cm}}
			\toprule
			& \multicolumn{2}{c}{ProSparse-LLaMA2-7B} & \multicolumn{2}{c}{TurboSparse-Mistral-Instruct} & \multicolumn{2}{c}{SparseQwen2-7B} \\
			\cmidrule(lr){2-3} \cmidrule(lr){4-5} \cmidrule(l){6-7}
			Sparsity & CUDA & CANN & CUDA & CANN & CUDA & CANN \\ \midrule
			20\% & 1.30$\times$ & 0.76$\times$ & 1.29$\times$ & 0.76$\times$ & 1.33$\times$ & 0.81$\times$ \\
			50\% & 1.90$\times$ & 1.23$\times$ & 1.90$\times$ & 1.18$\times$ & 1.95$\times$ & 1.32$\times$ \\
			80\% & 3.34$\times$ & 3.38$\times$ & 3.44$\times$ & 3.10$\times$ & 3.47$\times$ & 3.41$\times$ \\
			95\% & 4.67$\times$ & 30.67$\times$ & 4.92$\times$ & 22.85$\times$ & 4.91$\times$ & 28.46$\times$ \\
			\bottomrule
		\end{tabular}
	}
	\label{tab:npu-res}
\end{table}
}

\section{Limitations}
\label{sec:limitations}
\begin{itemize}
	\item Our method targets sparse models where non-linear activation functions naturally produce significant zero-valued outputs.
Extension to non-ReGLU-based architectures remains an important direction for future research. See \cref{sec:non-relu-based-llms} for detailed discussion.
	\item Batched processing inherently reduces sparsity levels in FFNs \cite{shrestha2025polar}, limiting achievable speedups and rendering the direct use of SVDPs ineffective for prefill and multi-token prediction. See \cref{sec:batched-sparsity} for details and possible solutions.
	\item The algorithm modifies only FFNs, leaving Attention blocks unchanged. Therefore, the E2E speedup is limited by FFN contribution to the total inference time. SVD-based contextual sparsity predictors for Attention layer is a future research direction.
	\item The calibration data should be drawn from the same distribution as the target application. See details in \cref{sec:sensitivity-analysis}
\end{itemize}

\section*{Conclusion}
This paper introduces a training-free framework for accelerating the inference of ReGLU-based large language models through contextual sparsity exploitation.
The proposed solution comprises two stages: (1) an offline calibration phase where sparsity predictors are constructed using data-aware truncated singular value decomposition of gate projection matrices with calibrated bias parameters, and (2) an online inference phase enabling efficient sparse LLM execution on CUDA and CANN devices.
Evaluation across three sparse language models demonstrates up to 1.8$\times$ end-to-end decoding acceleration with minimal accuracy degradation across diverse benchmarks, including code and mathematical reasoning tasks.

\section*{Acknowledgements}
We thank the MSU Laboratory of Computer Engineering for conducting all GPU-based experiments. 
These were performed exclusively by authors affiliated with Moscow State University; 
Huawei-affiliated co-authors, who declare no access to GPU-dependent aspects of this work due to U.S. export control regulations,
contributed solely to theoretical analysis, algorithm design, and CANN performance evaluation.

\section*{Impact statement}
\label{sec:impact}
This paper presents a training-free method for accelerating LLM inference through contextual sparsity, with particular relevance to edge device deployment. 
The main potential societal consequence of our work is AI's democratization.
By eliminating the need to train sparsity predictors, our approach significantly lowers the barrier to applying contextual sparsity techniques.

\bibliography{paper}
\bibliographystyle{plainnat}

\newpage
\appendix

\section{Theoretical analysis}
\label{sec:math-proofs}

\subsection{Predictor error bounds}
\label{sec:math-proofs-upper-error-bound}

Let $m,d > 0$ and $0 < r \le \min(m,d)$. Consider a weight matrix $W \in \R^{m\times d}$ 
and its rank-$r$ approximation $W_r \in \R^{m\times d}$ (e.g., truncated SVD), 
with bias vector $b \in \R^m$. For any input $x \in \R^d$, $x \neq 0$, we have:

\begin{align}
    A(x) &:= \relu(Wx), \label{eq:true_act} \\
    g_b(x) &:= \one[W_r x + b > 0], \label{eq:predictor} \\
    B_b(x) &:= A(x) \odot g_b(x). \label{eq:gated}
\end{align}

Throughout, we use the following notation:
\begin{align}
    \delta_i &\in \R^m, \quad i\in[m], \quad \text{denotes the $i$-th standard basis vector $(0,\ldots,\underbrace{1}_i,\ldots,0)^\top$}, \\
    t_+ &:= \max(t,0), \quad \text{the ReLU function (i.e., $t_+=0$ if $t<0$ and $t_+=t$ if $t\ge 0$)}.
\end{align}

We aim to derive deterministic bounds on the relative gating error
\begin{equation}
\mathcal{E}_b(x) := \frac{\norm{A(x) - B_b(x)}_2}{\norm{x}_2}
\end{equation}
in terms of the spectral error of $W_r$ and the bias $b$.

\paragraph{Exact error representation and monotonicity}
\begin{definition}[Residual operator]
Define the residual matrix and residual vector
\begin{equation}
\Delta W := W - W_r,
\qquad
\res(x) := \Delta W x .
\end{equation}
\end{definition}

\begin{proposition}[Exact representation of the gating error]\label{prop:exact_error}
For every $x \in \R^d$,
\begin{equation}\label{eq:exact_error}
A(x)-B_b(x)
= \relu(Wx) \odot \one[W_r x + b \le 0],
\end{equation}
and therefore
\begin{equation}\label{eq:exact_error_norm}
\norm{A(x)-B_b(x)}_2^2
= \sum_{i=1}^m \pos{(Wx)_i}^2 \, \one[(W_r x + b)_i \le 0].
\end{equation}
\end{proposition}

\begin{proof}
The claim follows immediately from the componentwise definition of $B_b(x)$.
\end{proof}

\begin{proposition}[Monotonicity w.r.t.\ biases]\label{prop:mono_b}
Let $b,b'\in\R^m$ and suppose $b'\le b$ componentwise. Then for all $x\in\R^d$,
\begin{equation}
\norm{A(x)-B_{b'}(x)}_2
\;\ge\;
\norm{A(x)-B_{b}(x)}_2.
\end{equation}
\end{proposition}

\begin{proof}
If $b'\le b$, then $\one[W_r x + b \le 0]\ \le\ \one[W_r x + b' \le 0]$ componentwise.
Using the identity \eqref{eq:exact_error} and the fact that $\relu(Wx)\ge 0$ componentwise, the $\ell_2$ norm can only increase when we replace $b$ by the smaller bias $b'$, which activates the mask on a superset of coordinates.
\end{proof}	

\begin{remark}
Using thresholds $\tau=-b$ clarifies the trade-off: increasing $\tau$ increases sparsity
(more gates are off) but never decreases the error; equivalently, increasing $b$
reduces sparsity and never increases the error.
\end{remark}

\paragraph{A deterministic bound via the shifted residual}

\begin{lemma}[Residual-based upper bound]\label{lem:shifted_residual}
For every $x \in \R^d$ and $b \in \R^m$,
\begin{equation}\label{eq:shift_tail}
\norm{A(x)-B_b(x)}_2
\;\le\;
\norm{\vpos{\Delta W x - b}}_2
\;\le\;
\norm{\Delta W x - b}_2.
\end{equation}
\end{lemma}

\begin{proof}
Let $z = Wx$, $z_r = W_r x$, and $\res := z - z_r = \Delta W x$.
Define the gated index set
\begin{equation}
I := \{ i \in [m] : (z_r + b)_i \le 0 \}.
\end{equation}
By definition,
\begin{equation}
\norm{A(x)-B_b(x)}_2^2 = \sum_{i\in I} \pos{z_i}^2.
\end{equation}
For $i\in I$ we have $z_{r,i} \le -b_i$, hence
$z_i = z_{r,i} + \res_i \le \res_i - b_i$.
By monotonicity of $t \mapsto \pos{t}$,
$\pos{z_i} \le \pos{\res_i - b_i}$.
Therefore
\begin{equation}
\sum_{i\in I} \pos{z_i}^2
\le
\sum_{i=1}^m \pos{\res_i - b_i}^2
=
\norm{\vpos{\res - b}}_2^2,
\end{equation}
which gives the first inequality.
The second inequality follows from $\norm{\vpos{u}}_2 \le \norm{u}_2$.
\end{proof}


\paragraph{Worst-case bounds}

Due to the previous result, the analysis of the nonlinear gating error reduces to a deterministic quantity depending only on the residual $\Delta W x$ and the bias $b$.
We therefore can study the worst-case value
\begin{equation}
\sup_{\norm{e}_2 \le R} \norm{\vpos{e-b}}_2,
\end{equation}
which characterizes the maximal gating error at residual energy $R$.

\subparagraph{Uniform bias}

\begin{theorem}[Worst-case bound for uniform bias]\label{thm:const_bias}
Let $b=\beta\mathbf 1$ with $\beta\in\R$. For all $R\ge 0$,
\begin{equation}
\sup_{\norm{e}_2 \le R} \norm{\vpos{e-\beta\mathbf{1}}}_2
=
\begin{cases}
\pos{R-\beta}, & \beta \ge 0,\\[3pt]
R - \beta \sqrt{m}, & \beta < 0.
\end{cases}
\end{equation}
\end{theorem}

\begin{proof}
\noindent\\
\textbf{Case 1: $\beta\ge 0$.}
Write $y:=\vpos{e-\beta\mathbf 1}\in\R^m_{\ge 0}$ and $t:=\norm{y}_2$.

Let $S:=\{i: y_i>0\}$ (if $t=0$ there is nothing to prove, otherwise $R > \beta$).

For $i\in S$ we have
$e_i=\beta+y_i$, hence
\begin{equation}
\norm{e}_2^2 \;\ge\; \sum_{i\in S}(\beta+y_i)^2
\;=\; \beta^2|S| + 2\beta\sum_{i\in S}y_i + \sum_{i\in S}y_i^2.
\end{equation}
Since $y\ge 0$, we have $\sum_{i\in S}y_i=\norm{y}_1\ge \norm{y}_2=t$, and also $|S|\ge 1$.
Therefore
\begin{equation}
\norm{e}_2^2 \;\ge\; \beta^2 + 2\beta t + t^2 \;=\; (t+\beta)^2,
\end{equation}
so $\norm{e}_2\ge t+\beta$. Under $\norm{e}_2\le R$ this yields $t\le R-\beta$, i.e.,
$\norm{\vpos{e-\beta\mathbf 1}}_2\le (R-\beta)_+$.

Equality holds by taking
\begin{align}
R\le\beta: &\quad e=0, \\
R>\beta: &\quad e=R\delta_1.
\end{align}

\medskip
\noindent\textbf{Case 2: $\beta<0$.}
Write $\beta=-c$ with $c>0$. Then for any $e$,
\begin{equation}
\norm{\vpos{e-\beta\mathbf 1}}_2
=\norm{\vpos{e+c\mathbf 1}}_2
\le \norm{e+c\mathbf 1}_2
\le \norm{e}_2 + c\norm{\mathbf 1}_2
\le R + c\sqrt m
= R-\beta\sqrt m.
\end{equation}
Equality holds for $e=\frac{R}{\sqrt m}\mathbf 1$.
\end{proof}


\subparagraph{General bias vector}
\begin{theorem}[Worst-case bound for general bias]\label{thm:general_bias}
Let $b\in\R^m$ and $R\ge 0$. Then
\begin{equation}
	\sup_{\norm{e}_2\le R} \norm{\vpos{e-b}}_2 =
	\begin{cases}
	\max_{i\in[m]} (R - b_i)_+, &  b \ge 0 \;\; \text{componentwise}, \\[1em]
	R + \norm{b}_2, &  b < 0 \;\; \text{componentwise}.
	\end{cases}
\end{equation}
\end{theorem}	

\begin{proof}
\noindent\\
\textbf{Case 1: $b < 0$.} 
\begin{equation}\norm{(e-b)_+}_2\le\norm{e-b}_2\le R+\norm{b}_2.\end{equation}
Tightness: $e=\frac{R}{\norm{b}_2}(-b)$ gives exactly $R+\norm{b}_2$.

\medskip
\noindent\textbf{Case 2: $b \ge 0$.} 
For $i^\star=\arg\min_i b_i$, let $b_c=b_{i^\star}\mathbf{1}$. Then $(e-b)_+\le(e-b_c)_+$ componentwise, so
\begin{equation}
\norm{(e-b)_+}_2\le\norm{(e-b_c)_+}_2\overset{\eqref{thm:const_bias}}{\le}(R-b_{i^\star})_+=\max_i(R-b_i)_+.
\end{equation}
Tightness: $e=R\delta_{i^\star}$ for $i^\star=\arg\max_i(R-b_i)_+$ achieves equality.
\end{proof}

\begin{remark}
Theorem~\ref{thm:general_bias} identifies two regimes in which the worst-case residual
can be characterized exactly: a fully negative bias, where the bound grows linearly
with $\|b_-\|_2$, and a fully nonnegative bias, where the optimum concentrates on a
single coordinate.
\end{remark}

\begin{corollary}[Main deterministic bound]\label{thm:main}
Let $b\in\R^m$ and $R\ge 0$ be such that $\|\Delta W x\|_2 \le R$.
Then for all $x \in \R^d$,
\begin{equation}
\mathcal{E}_b(x) =
\frac{\norm{A(x)-B_b(x)}_2}{\norm{x}_2}
\le
\begin{cases}
\displaystyle \frac{\max_{i\in[m]} (R - b_i)_+}{\|x\|_2}, & b \ge 0 \;\; \text{componentwise},\\[0.75em]
\displaystyle \frac{R + \|b\|_2}{\|x\|_2}, & b < 0 \;\; \text{componentwise}.
\end{cases}
\end{equation}
\end{corollary}

\begin{corollary}[Truncated SVD]\label{cor:svd}
	By Eckart--Young--Mirsky, if $W_r$ is the rank-$r$ truncated SVD of $W^\mathrm{gate}$, then
$\|\Delta W\|_2 = \sigma_{r+1}(W^\mathrm{gate}) = \sigma_{r+1}$, so for all $x\in\R^d$,
\begin{equation}
\mathcal{E}_b(x)
\le
\begin{cases}
\displaystyle \max_{i\in[m]} \left(\sigma_{r+1} - \frac{b_i}{\norm{x}_2} \right)_+, & b \ge 0 \;\; \text{componentwise},\\[0.75em]
\displaystyle \sigma_{r+1} + \frac{\norm{b}_2}{\norm{x}_2}, & b < 0 \;\; \text{componentwise}.
\end{cases}
\end{equation}
\end{corollary}

\begin{remark}
	While rather loose in practice due to data distribution, this bound cleanly separates approximation error $\sigma_{r+1}$ from gating bias $b$, revealing bias addition's influence on residuals.
\end{remark}

\subsection{Greedy threshold selection}
\label{sec:math-proofs-greedy-threshold}

\begin{algorithm*}
	\caption{Greedy Bias Calibration (Offline Phase)}
	\label{alg:bias-calibration}
	\begin{algorithmic}
	\STATE {\bfseries Input:}  \\
	$\quad W^\mathrm{gate}$, $W^\mathrm{up}\in\mathbb{R}^{D\times d}$, $W^\mathrm{down}\in\mathbb{R}^{d\times D}$ -- weight matrices,\\
	$\quad \mathcal D=\{x^{(t)} \in\mathbb{R}^{d} \}_{t=1}^T$ -- calibration dataset of $T\in\mathbb{N}$ tokens, \\
	$\quad A, B$ -- low-rank approximation from \cref{alg:svd-data-whitening}, \\
	$\quad s\in[0;1]$ -- desired sparsity level, $\eta\in\mathbb{N}$ -- discrete step size.

	\STATE {\bfseries Output:}  $b\in\mathbb{R}^{D}$

	\STATE $w_i^{(t)} \leftarrow \left|\mathrm{ReLU}\left(W^\mathrm{gate}_{i, :}x^{(t)}\right) \cdot W^\mathrm{up}_{i, :}x^{(t)}\right|^2 \left\| W^\mathrm{down}_{:, i} \right\|_2^2$ 
	\hfill \textit{\small (Compute neurons' importance)}

	\STATE $s_{i}^{(t)} \leftarrow (AB x^{(t)})_i$
	\hfill \textit{\small (Compute predictor's outputs)}

	\FOR{$i=1$ {\bfseries to} $D$}
		\STATE $s_{i,(1)}\le\cdots\le s_{i,(T)}$
		\hfill \textit{\small (Sort \underline{samples} in nondecreasing order)}

		\STATE $C_i(k)=\begin{cases}
			\sum_{j=1}^k w_{i,(j)}, & k=1,\dots,T \\
			C_i(T), & T<k<T+\eta\\
			+\infty, & k=T+\eta
		\end{cases}$
		\hfill \textit{\small (Compute cumulative costs with padding)} 

		\STATE $k_i\leftarrow \min\{k: w_{i,(k)}>0\}, \quad \tau_i\leftarrow s_{i, (k_i)}$
		\hfill \textit{\small (Initialize pointers and thresholds)}
	\ENDFOR
	
	\WHILE{$\frac{1}{DT}\sum_{i=1}^D\sum_{t=1}^T\mathbf{1}[s_i^{(t)} \le \tau_i] < s$}
		\STATE $i^\star\leftarrow\arg\min_{i} \left[C_i(k_i + \eta)-C_i(k_i)\right]$
		\hfill \textit{\small (Greedy step)}
		\STATE $k_{i^\star}\leftarrow \min(k_{i^\star} + \eta,\: T), \quad \tau_{i^\star}\leftarrow s_{i^\star, (k_{i^\star})}$
		\hfill \textit{\small (Update pointers and thresholds)}
	\ENDWHILE

	\STATE $b\leftarrow -\tau$
	\STATE \textbf{return} $b$
	
	\end{algorithmic}
\end{algorithm*}

\paragraph{Greedy threshold calibration under a fixed sparsity budget}
\label{sec:greedy_calibration}

We study a data-dependent calibration rule for componentwise thresholds
$\tau\in\R^m$ under a fixed \emph{drop budget} $K$ (equivalently, fixed sparsity).
Throughout, the gate is
\begin{equation}
g_\tau(x)=\one[AB x > \tau],
\end{equation}
where $AB$ denotes any rank-$r$ approximation of $W^\mathrm{gate} \in \R^{m \times d}$; the specific choice used in our experiments is given in \cref{subsec:svd-based-decomposition}.

\paragraph{Proxy scores and damage weights}
\label{subsec:scores_weights}

Let $\mathcal D=\{x^{(t)}\}_{t=1}^T$ be a dataset and write $[m]=\{1,\dots,m\}$,
$[T]=\{1,\dots,T\}$.
For each coordinate $i\in[m]$ and sample $t\in[T]$ define the proxy score
\begin{equation}
s_i^{(t)} := (AB x^{(t)})_i,
\end{equation}
and define the \emph{marginal damage weight}
\begin{equation}
\label{eq:damage_weight}
w_i^{(t)}
:=
\left|\mathrm{ReLU}\left(W^\mathrm{gate}_{i, :}x^{(t)}\right) \cdot W^\mathrm{up}_{i, :}x^{(t)}\right|^2 \left\| W^\mathrm{down}_{:, i} \right\|_2^2
\qquad (\ge 0).
\end{equation}
Given thresholds $\tau$, the empirical output error incurred by dropping
$(i,t)$ whenever $s_i^{(t)}\le \tau_i$ is
\begin{equation}
\label{eq:empirical_loss}
\mathcal L(\tau)
=
\sum_{t=1}^T\sum_{i=1}^m
w_i^{(t)}\;\one[s_i^{(t)}\le \tau_i].
\end{equation}

\begin{remark}[Origin of the weight \eqref{eq:damage_weight}]
\label{rem:origin_weight}
Consider the standard gated FFN decomposition
\begin{equation}
z(x)=\relu(W^\mathrm{gate}x)\odot (W^\mathrm{up}x),
\qquad
h(x)=W^\mathrm{down}z(x).
\end{equation}
Then $h(x)$ decomposes into rank-one terms:
\begin{equation}
h(x)=\sum_{i=1}^m z_i(x)\,W^\mathrm{down}_{:,i},
\quad
z_i(x)=\relu(W^\mathrm{gate}_{i,:} x) \cdot W^\mathrm{up}_{i,:}x.
\end{equation}
Dropping the $i$-th hidden coordinate removes exactly the contribution
$v_i(x)=z_i(x) W^\mathrm{down}_{:, i}$, hence
\begin{equation}
\|v_i(x)\|_2^2
=
\left|\mathrm{ReLU}\left(W^\mathrm{gate}_{i, :}x\right) \cdot W^\mathrm{up}_{i, :}x\right|^2 \left\| W^\mathrm{down}_{:, i} \right\|_2^2,
\end{equation}
which matches \eqref{eq:damage_weight} for $x=x^{(t)}$.
\end{remark}

\paragraph{Prefix formulation and fixed drop budget}
\label{subsec:prefix_budget}

Fix $i\in[m]$. Sort samples in nondecreasing order of proxy scores:
\begin{equation}
s_{i,(1)} \le s_{i,(2)} \le \cdots \le s_{i,(T)},
\end{equation}
and let $w_{i,(j)}$ denote the corresponding permuted weights.
Define
\begin{equation}
C_i(k):=\sum_{j=1}^k w_{i,(j)},\quad C_i(0)=0,
\qquad
\Delta_i(k):=C_i(k)-C_i(k-1)=w_{i,(k)}.
\end{equation}

For each coordinate $i$, setting $\tau_i$ so exactly $k_i$ samples satisfy $s_i^{(t)}\le \tau_i$ incurs cost $C_i(k_i)$. The total loss is thus
\begin{equation}
\mathcal L(\tau)=\sum_{i=1}^m C_i(k_i), \qquad k_i := \#\{t:\ s_i^{(t)}\le \tau_i\}\in\{0,\dots,T\}.
\end{equation}

With fixed drop budget $K\in\{0,1,\dots,mT\}$, threshold selection reduces to
\begin{equation}
\label{eq:fixed_budget_problem}
\min_{\substack{k_i\in\{0,\dots,T\}\\ \sum_{i=1}^m k_i=K}}
\ \sum_{i=1}^m C_i(k_i).
\end{equation}

\paragraph{Greedy optimality}
\label{subsec:greedy_and_nearopt}

\begin{definition}[Global optimality]
	\label{def:global_opt}
	A feasible solution $k^\star=(k_1^\star,\dots,k_m^\star)$ with $\sum_i k_i^\star=K$ 
	is \emph{globally optimal} for \eqref{eq:fixed_budget_problem} if
	\begin{equation}
	\sum_{i=1}^m C_i(k_i^\star) \le \sum_{i=1}^m C_i(k_i)
	\quad\text{for all feasible $k=(k_1,\dots,k_m)$.}
	\end{equation}
\end{definition}

\begin{assumption}[Monotone marginals within each coordinate]
\label{ass:monotone_marginals}
For every $i\in[m]$, the marginal increments are nondecreasing:
\begin{equation}
\Delta_i(1)\le \Delta_i(2)\le \cdots \le \Delta_i(T).
\end{equation}
\end{assumption}

\begin{theorem}[Greedy is globally optimal]
\label{th:greedy_optimal}
Initialize $k_i=0$ for all $i$.
For $\ell=1,2,\dots,K$ (drop budget) choose
\begin{equation}
i_\ell \in \arg\min_{i:\,k_i<T}\ \Delta_i(k_i+1),
\qquad\text{and update}\qquad
k_{i_\ell}\leftarrow k_{i_\ell}+1.
\end{equation}
Under~\cref{ass:monotone_marginals} the resulting $k$ is globally optimal for \eqref{eq:fixed_budget_problem}.

Moreover, it can be realized by thresholds $\tau$ via
\begin{equation}
\tau_i \in [\,s_{i,(k_i)},\, s_{i,(k_i+1)}\,),
\end{equation}
with conventions $s_{i,(0)}:=-\infty$ and $s_{i,(T+1)}:=+\infty$.
\end{theorem}

\begin{proof}
The problem is equivalent to selecting exactly $K$ increments $\Delta_i(j)$ while respecting prefix constraints: selecting $\Delta_i(j)$ requires all $\Delta_i(1),\dots,\Delta_i(j-1)$.

The greedy algorithm selects the smallest available ``next increment'' $\Delta_i(k_i+1)$ at each step. Let $G$ be greedy's sequence of $K$ increments, and $O$ any optimal sequence.

Suppose $G\neq O$. Let $a=\Delta_p(k_p+1)$ be the first (chronologically) greedy increment not in $O$. Since $|G|=|O|=K$, there exists $b=\Delta_q(j)\in O\setminus G$.

Choose $b$ as the \emph{last} increment of coordinate $q$ in $O$ (maximal $j$). Removing $b$ preserves $q$'s prefix constraint. Adding $a$ extends $p$'s prefix, also preserving feasibility.

By greedy choice, $a\le\Delta_q(k_q+1)$ where $k_q$ counts greedy's selections in $q$ up to that point. Since $b\notin G$ but $O$ includes it, $j\ge k_q+1$. By \cref{ass:monotone_marginals} (nondecreasing increments), $\Delta_q(k_q+1)\le\Delta_q(j)=b$.

Thus $a\le b$, so swapping $b\leftrightarrow a$ preserves feasibility and does not increase cost. Repeating such exchanges transforms $O$ into $G$, proving optimality.

The threshold realization follows from the definition of $k_i$.
\end{proof}

\begin{remark}[Role of rank-$r$ approximation and empirical validation]
\label{rem:svd_proxy}

\Cref{th:greedy_optimal} requires monotone marginal increments after sorting by proxy scores (\cref{ass:monotone_marginals}).
This is \emph{not guaranteed}, since proxy scores $s_i^{(t)} = (ABx^{(t)})_i$ depend only on rank-$r$ approximation $AB\approx W^\mathrm{gate}$,
while damage, defined above, also involves $W^\mathrm{up}$ and $W^\mathrm{down}$.

If $AB$ approximates $W^\mathrm{gate}$ well \emph{on the input data distribution}, then $(ABx)_i\approx (W^\mathrm{gate}x)_i$, so proxy sorting should align with true damage ordering.
In practice, we perform greedy selection over non-overlapping $\eta$-step sums to mitigate data noise:
\begin{equation}
S_i^{(m)}:=\sum_{j=(m-1)\eta+1}^{m\eta} w_{i,(j)},\quad m=1,\dots,\lfloor T/\eta\rfloor.
\end{equation}
We validate their monotonicity using normalized Kendall tau distance between group sums and their position indices:
\begin{equation}
\tau_i^{\mathrm{K}} := 1 - \frac{2K_d\bigl(\{S_i^{(m)}\}_{m=1}^M,\ \{m\}_{m=1}^M\bigr)}{M(M-1)},
\end{equation}
where $
K_d(\tau_1,\tau_2)=\#\bigl\{(m_1,m_2):\ m_1<m_2,\ 
[\tau_1(m_1)\lessgtr\tau_1(m_2)\land\tau_2(m_1)\gtrless\tau_2(m_2)]
\bigr\}
$ counts discordant pairs.

The use of Kendall tau distance is justified by the fact that it directly quantifies \emph{ordinal consistency} required by \cref{ass:monotone_marginals}: for all $m_1 < m_2$, we expect $S_i^{(m_1)} \leq S_i^{(m_2)}$.
Formally, $\tau_i^{\mathrm{K}}$ measures the fraction of concordant pairs while properly accounting for ties -- the precise condition for greedy optimality in \cref{th:greedy_optimal}.
On sample data used to calibrate biases, \Cref{fig:kendall-distance} shows stable Kendall $\tau_i^{\mathrm{K}}$ in $[0.9, 1]$, confirming proxy sorting $ABx \approx W^\mathrm{gate}x$ and the suitability of our greedy approach.

\begin{figure}
	\centering
	\includegraphics[width=0.65\textwidth]{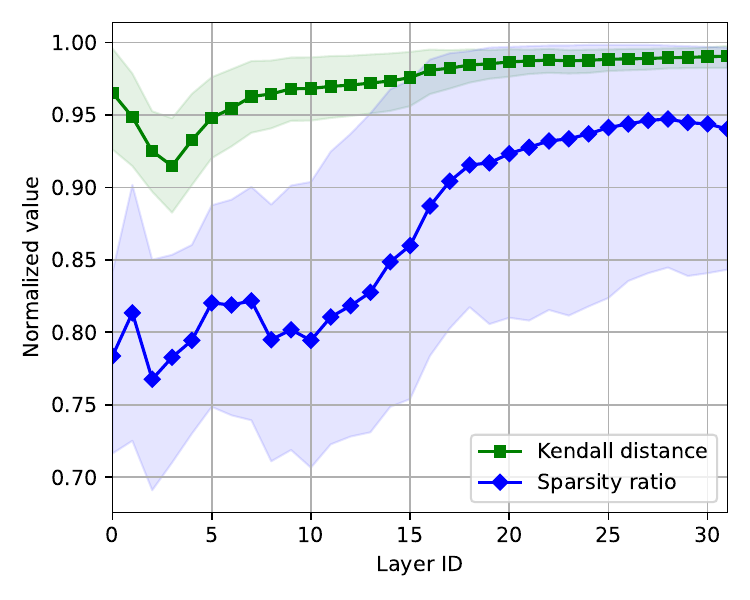}
	\caption{Mean Kendall $\tau_i^{\mathrm{K}}$ across neurons for proxy-sorted $\eta$-step sums from ProSparse-LLaMA2-7B.
	The x-axis shows layer number. For clarity, layer-wise sparsity ratios are depicted. The shaded regions denote the 10th to 90th percentile interval, illustrating the central 80\% of the data distribution. }
	\label{fig:kendall-distance}
\end{figure}

\end{remark}

\begin{remark}[Knapsack reformulation and practical optimality]
\label{ref:knapsack}

Problem \eqref{eq:fixed_budget_problem} is a $0$--$1$ knapsack problem with $mT$ items $\{\Delta_i(j)\}_{i=1}^m{}_{j=1}^T$, capacity $K$, and prefix constraints that selection of $\Delta_i(j)$ requires prior selection of $\Delta_i(1),\dots,\Delta_i(j-1)$.
While standard dynamic programming requires $O(mTK)$ time, our greedy algorithm runs in $O((K + m)\log m)$ via priority queue and is \emph{exactly optimal} under \cref{ass:monotone_marginals}, yielding superior efficiency.
\end{remark}

\paragraph{Time and space complexity}
For one FFN block with $D$ neurons and $T$ calibration tokens of hidden size $d$, computing and storing the neuron importances is a single FFN forward pass, $O(dDT)$, which is normally already available from the calibration pass that the SVD stage requires. 
Sorting the activations per neuron is $O(DT\log T)$. 
The calibration itself is the knapsack problem of \cref{ref:knapsack}, with $DT$ items and a drop budget $K<DT$ set by the target sparsity, solved greedily in $O((K+D)\log D)$ with a priority queue. 
In total $O(dDT+DT\log T + (K+D)\log D)$ time and $O(DT)$ additional memory.


\section{Sparse FFN implementation}
\label{sec:inference-executors-details}

\Cref{fig:sparse-inference} illustrates the sparse inference pipeline, where speedup is determined by the token-dependent sparsity ratio.
\begin{figure*}[t]
    \centering
    \includegraphics[width=0.97\textwidth]{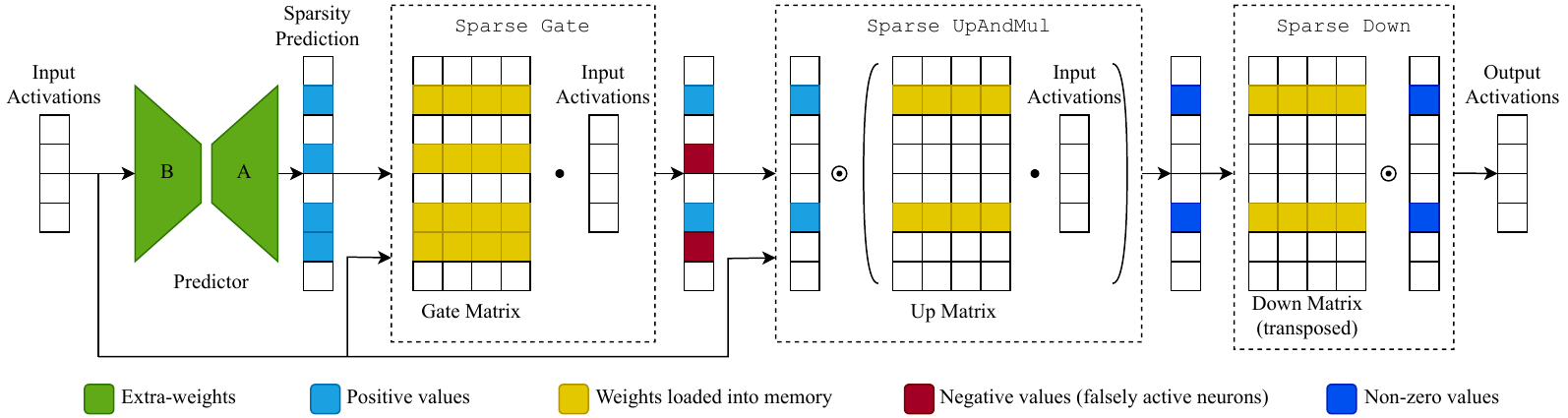}
    \caption{Sparse inference using SVD-based sparsity predictors. Neurons with positive predictor values are considered active during gate projection. Gate and up projections are executed sequentially to increase sparsity.}
    \label{fig:sparse-inference}
\end{figure*}

Prediction latency can be hidden via asynchronous computation, where predictor execution time is overlapped with data transfer from global memory or inter-process communication under tensor parallelism. 
Unlike Deja Vu's \cite{liu2023deja} asynchronous pipeline, which leverages Ampere's \texttt{cuda::memcpy\_async} operations \cite{cuda-c-programming-guide}, our approach targets consumer-grade GPUs with compute capability 7.0 for broader applicability.
Predictors are inserted at the beginning of each FFN and sequentially generate sparsity patterns for their respective gate projections.
This design ensures minimal integration overhead with existing frameworks.

\paragraph{Custom CUDA kernels for sparse FFN}
We implement three CUDA kernels (\cref{alg:sparse-inference}). 
\begin{algorithm}[tb]
	\caption{Sparse FFN Forward Pass on CUDA (Decoding Phase)}
	\label{alg:sparse-inference}
	\begin{algorithmic}
		\STATE \textbf{Input:} $W^\mathrm{gate}, W^\mathrm{up} \in \mathbb{R}^{D \times d}$, $W^\mathrm{down} \in \mathbb{R}^{d \times D}$, $x \in \mathbb{R}^d$, $A \in \mathbb{R}^{D \times r}$, $B \in \mathbb{R}^{r \times d}$, $b \in \mathbb{R}^D$
		\STATE \textbf{Output:} $x \in \mathbb{R}^d$
		\STATE $m \gets ABx + b$ \hfill \textit{\small (Predict sparsity pattern)}
		\STATE $m \gets \textsc{SparseGateKernel}(W^\mathrm{gate},\,x,\,m)$
		\STATE $x \gets \textsc{SparseUpAndMulKernel}(W^\mathrm{up},\,x,\,m)$
		\STATE $x \gets \textsc{SparseDownKernel}(W^\mathrm{down},\,x)$
		\STATE \textbf{return} $x$
	\end{algorithmic}
\end{algorithm}
The predictor’s output $m \in \mathbb{R}^D$ highlights active rows of $W^\mathrm{gate}$. Throughout gate and up projections each CUDA block processes one neuron, first verifying its activity via $m$.
Gate outputs are stored back into $m$ for reuse.
Key optimizations:
\begin{itemize}
    \item \textbf{Memory}: Column-major storage of $W^\mathrm{down}$ for coalesced memory access.
    \item \textbf{Precision}: Full precision register operations; half precision for memory reads and writes.
    \item \textbf{Determinism}: Avoid atomic additions via shared-memory partial sums in down projection.
\end{itemize}

Our kernel design diverges from existing ProSparse \cite{song2024prosparse} and Deja Vu \cite{liu2023deja} implementations. 
Both prior approaches rely on atomic addition operations within the down projection phase, introducing non-determinism and potential precision instability during floating-point accumulation. 
To address these limitations, our algorithm employs partial sum reduction implemented via shared memory. 
This technique enforces a predefined summation order and utilizes a hierarchical reduction strategy. 
Consequently, our approach guarantees deterministic outputs while simultaneously achieving enhanced numerical stability through reduced computational error propagation.
We postpone kernel fusion due to diminishing returns, prioritizing maintainability. 

\paragraph{NPU solution for sparse FFN}
We implement our NPU solution for sparse FFN, detailed in \cref{alg:sparse-npu}, leveraging high-level PyTorch instructions.  
\begin{algorithm}[tb]
	\caption{Sparse FFN Forward Pass on CANN (Decoding Phase)}
	\label{alg:sparse-npu}
	\begin{algorithmic}
		\STATE \textbf{Input:} $W^{\text{gate}}, W^{\text{up}} \in \mathbb{R}^{D \times d}$, $W^{\text{down}} \in \mathbb{R}^{d \times D}$, $x \in \mathbb{R}^d$, $A \in \mathbb{R}^{D \times r}$, $B \in \mathbb{R}^{r \times d}$, $b \in \mathbb{R}^D$
		\STATE \textbf{Output:} $x \in \mathbb{R}^d$
		\STATE $m \gets ABx + b > 0$ \hfill \textit{\small (Predict sparsity pattern)}
		\STATE $p \gets  \textsc{nonzero}(m)$ \hfill \textit{\small (Obtain indices)}
		\STATE $g \gets \textsc{selrows}(W^{\text{gate}},\,p)$ \hfill \textit{\small (Select rows of $W^{\text{gate}}$ indexed by $p$)}
		\STATE $m_1 \gets \textsc{matmul}({g},\,x)$
		\STATE $u \gets \textsc{selrows}(W^{\text{up}},\,p)$ \hfill \textit{\small (Select rows of $W^{\text{up}}$ indexed by $p$)}
		\STATE $m_2 \gets \textsc{matmul}({u},\,x)$
		\STATE $m_3 \gets \textsc{relu}(m_1) \cdot m_2$
		\STATE $d \gets \textsc{selcols}(W^{\text{down}},\,p)$ \hfill \textit{\small (Select columns of $W^{\text{down}}$ indexed by $p$)}
		\STATE $x \gets \textsc{matmul}({d},\,m_3)$
		\STATE \textbf{return} $x$
	\end{algorithmic}
\end{algorithm}
The predictor’s output $m \in \mathbb{R}^D$ is used to construct smaller dense submatrices by retrieving the corresponding rows or columns. 
This approach enables highly efficient execution by utilizing the dedicated matrix multiplication accelerator unit (cube) inherent to the CANN NPU architecture \cite{chen2019davinci}.
$W^\mathrm{down}$ matrix is stored in column-major format.
\cref{alg:sparse-npu} is written for the ReGLU case for clarity; the dReLU variant differs only in the additional ReLU on $m_2$.

\section{Detailed accuracy results}
\label{sec:detailed-acc-results}

{
	\setlength{\tabcolsep}{3pt}
	\begin{table}[tb]
		\caption{ProSparse-LLaMA2-7B evaluation results, $r=256$.}
		\centering
		{
			\footnotesize
			\begin{tabular}{lccccccccccc}
				\toprule
				Method				& $s$  &  TriviaQA & ARC-E   & ARC-C  & HumanEval & MBPP & GSM8K & BBH & CMMLU & Avg. & E2E \\
				& 	   & acc, \%   & acc, \% &acc, \% &pass@1, \% &pass@1, \% &acc, \% &acc, \% &acc, \% & Score & Speedup \\ \midrule
				Dense			& -    &  62.94    & 75.42   & 53.58  & 17.68     & 22.59 & 11.45 & 35.85 & 31.84 & 38.92& 1.00$\times$ \\ \midrule
				\multirow{4}{*}{SVDP} 			& 0.4  & 62.38	&75.29	&53.67	&17.68	&22.38	&10.61	&35.46	&31.81	&38.66			& 1.61$\times$ \\ 
				& 0.5   & 62.74	&75.34	&53.67	&17.68	&22.07	&11.22	&35.58	&31.94	&38.78 	&  1.64$\times$  \\ 
				& 0.6   & 62.90	&75.29	&53.75	&18.29	&21.97	&10.54	&35.54	&31.91	&38.77   & 1.67$\times$  \\ 
				& 0.7  & 62.04	&75.42	&53.58	&17.07	&21.87	&11.22	&35.43	&31.90	&38.57   & 1.71$\times$  \\ \bottomrule
			\end{tabular}
			}
			\label{tab:llama2-details-hp}
		\end{table}
		}
		{
			\setlength{\tabcolsep}{3pt}
			\begin{table}
	\caption{TurboSparse-Mistral-Instruct evaluation results, $r=352$.}
	\centering
	{
		\footnotesize
		\begin{tabular}{lccccccccccc}
			\toprule
			Method				& $s$  &  TriviaQA & ARC-E   & ARC-C  & HumanEval & MBPP & GSM8K & BBH & CMMLU & Avg. & E2E \\
			& 	   & acc, \%   & acc, \% &acc, \% &pass@1, \% &pass@1, \% &acc, \% &acc, \% &acc, \% & Score & Speedup \\ \midrule
			Dense			& -    &  48.58 & 87.21 & 71.67 & 34.15 & 42.20 & 61.71 & 42.35 & 46.52 & 54.30 & 1.00$\times$ \\ \midrule
			\multirow{4}{*}{SVDP} 			& 0.4   & 45.46	&87.08&	71.67&	32.93&	42.09&	61.49&	42.46&	46.55&	53.72			& 1.67$\times$ \\ 
			& 0.5   & 43.26	&87.04&	71.76&	33.54&	42.71&	61.41&	42.32&	46.47&	53.56 & 1.73$\times$ \\ 										
			& 0.6   & 40.63	&87.08&	72.10&	33.54&	42.09&	61.41&	42.19&	46.43&	53.18 & 1.79$\times$  \\ 
			& 0.7   & 39.21	&87.21&	72.10&	37.20&	42.09&	60.96&	42.33&	46.45&	53.44 & 1.84$\times$  \\ \bottomrule
		\end{tabular}
		}
		\label{tab:mistral-details-hp}
	\end{table}
	}
	{
		\setlength{\tabcolsep}{3pt}
		\begin{table}[tb]
			\caption{SparseQwen2-7B evaluation results, $r=512$.}
			\centering
			{
				\footnotesize
				\begin{tabular}{lccccccccccc}
					\toprule
					Method				& $s$  &  TriviaQA & ARC-E   & ARC-C  & HumanEval & MBPP & GSM8K & BBH & CMMLU & Avg. & E2E \\
					& 	   & acc, \%   & acc, \% &acc, \% &pass@1, \% &pass@1, \% &acc, \% &acc, \% &acc, \% & Score & Speedup \\ \midrule
					Dense			& -    & 50.76 &	91.92&	83.87	&64.02&	46.3	&77.1	&48.42&	64.35&	65.84 & 1.00$\times$ \\ \midrule
					\multirow{4}{*}{SVDP} 			& 0.4   & 50.87&	91.92&	83.87&	62.80&	44.56&	76.50&	48.40&	64.35&	65.41			& 1.78$\times$ \\ 
					& 0.5   & 49.96&	91.92&	83.87&	62.80&	43.53&	76.80&	48.40&	64.35&	65.20 & 1.81$\times$ \\ 										
					& 0.6   & 48.82&	91.92&	83.87&	63.41&	43.53&	76.80&	48.34&	64.35&	65.13&  1.85$\times$  \\ 
					& 0.7   & 47.42&	91.92&	83.87&	62.20&	41.99&	75.82&	48.36&	64.35&	64.49&  1.86$\times$  \\ \bottomrule
				\end{tabular}
				}
				\label{tab:qwen2-details-hp}
			\end{table}
			}
\cref{tab:llama2-details-hp}, \cref{tab:mistral-details-hp}, \cref{tab:qwen2-details-hp} provide detailed evaluation results for three sparse LLMs equipped with our SVD-predictors.
\cref{tab:components-breakdown-llama2}, \cref{tab:components-breakdown-mistral}, \cref{tab:components-breakdown-qwen2} provide a detailed ablation study of SVD sparsity estimators for three sparse LLMs.
\begin{table*}[tb]
    \centering
	\caption{SVD sparsity estimator ablation study for ProSparse-LLaMA2-7B.}		
    \label{tab:components-breakdown-llama2}
    \footnotesize
	\begin{tabular}{p{2.5cm}ccccccccc}
		\toprule
		& TriviaQA & ARC-E & ARC-C & HumanEval & MBPP & GSM8K & BBH & CMMLU & Avg. \\
		\midrule
		Dense baseline & 62.94 & 75.42 & 53.58 & 17.68 & 22.59 & 11.45 & 35.85 & 31.84 & 38.92 \\
		\midrule
		\multicolumn{10}{l}{\textit{Parallel pipeline}} \\
		\hspace{1em}Naive SVD & 58.19 & 69.91 & 49.32 & 10.98 & 13.35 & 6.60 & 33.98 & 30.28 & 34.08 \\
		\hspace{1em}+ Data whitening & 53.85 & 75.34 & 53.75 & 16.46 & 19.92 & 10.16 & 35.46 & 31.83 & 37.10 \\
		\hspace{1em}+ Bias calibration & 62.66 & 75.38 & 53.58 & 17.68 & 22.18 & 10.54 & 35.86 & 31.85 & 38.72 \\
		\midrule
		\multicolumn{10}{l}{\textit{Sequential pipeline}} \\
		\hspace{1em}Naive SVD & 34.53 & 3.58 & 4.35 & 0.61 & 0.21 & 1.90 & 11.12 & 21.23 & 9.69 \\
		\hspace{1em}+ Data whitening & 27.58 & 75.21 & 53.33 & 8.54 & 11.40 & 6.90 & 33.71 & 31.34 & 31.00 \\
		\hspace{1em}+ Bias calibration & 62.74 & 75.34 & 53.67 & 17.68 & 22.07 & 11.22 & 35.58 & 31.94 & 38.78 \\
		\bottomrule
	\end{tabular}
\end{table*}

\begin{table*}[tb]
    \centering
	\caption{SVD sparsity estimator ablation study for SparseQwen2-7B.}		
    \label{tab:components-breakdown-qwen2}
    \footnotesize
	\begin{tabular}{p{2.5cm}ccccccccc}
		\toprule
		& TriviaQA & ARC-E & ARC-C & HumanEval & MBPP & GSM8K & BBH & CMMLU & Avg. \\
		\midrule
		Dense baseline &50.76 &	91.92	&83.87&	64.02&	46.30&	77.10&	48.42&	64.35&	65.84		\\
		\midrule
		\multicolumn{10}{l}{\textit{Parallel pipeline}} \\
		\hspace{1em}Naive SVD & 46.24	&91.92	&83.87	&53.66	&34.39&	74.00	&48.00&	64.35	&62.05		\\
		\hspace{1em}+ Data whitening & 42.92	&91.92&	83.87&	61.59&	43.02&	75.89&	48.37&	64.35&	63.99		\\
		\hspace{1em}+ Bias calibration & 50.20	&91.92&	83.87&	62.20&	45.28&	76.65&	48.45&	64.35&	65.37		\\
		\midrule
		\multicolumn{10}{l}{\textit{Sequential pipeline}} \\
		\hspace{1em}Naive SVD & 40.03 &	91.92	&83.87&	12.20&	9.03&	63.00&	48.43&	64.35&	51.60		\\
		\hspace{1em}+ Data whitening & 30.00&	91.92&	83.87&	57.32&	37.78&	72.93&	48.22&	64.35&	60.80		\\
		\hspace{1em}+ Bias calibration & 49.96&	91.92&	83.87&	62.80&	43.63&	76.80&	48.40&	64.35&	65.22		\\
		\bottomrule
	\end{tabular}
\end{table*}

\begin{table*}[tb]
    \centering
	\caption{SVD sparsity estimator ablation study for TurboSparse-Mistral-Instruct.}		
    \label{tab:components-breakdown-mistral}
    \footnotesize
	\begin{tabular}{p{2.5cm}ccccccccc}
		\toprule
		& TriviaQA & ARC-E & ARC-C & HumanEval & MBPP & GSM8K & BBH & CMMLU & Avg. \\
		\midrule
		Dense baseline &48.58	&87.21&	71.67&	34.15&	42.20&	61.71&	42.35&	46.52&	54.30		\\
		\midrule
		\multicolumn{10}{l}{\textit{Parallel pipeline}} \\
		\hspace{1em}Naive SVD & 		 41.40&	87.04&	71.76&	24.39&	38.91&	60.20&	42.13&	46.34&	51.52		\\
		\hspace{1em}+ Data whitening &   37.98&	86.57&	71.25&	37.20&	41.79&	60.88&	42.27&	46.18&	53.02		\\
		\hspace{1em}+ Bias calibration & 46.21&	87.12&	71.76&	36.59&	41.89&	62.24&	42.44&	46.45&	54.34		\\
		\midrule
		\multicolumn{10}{l}{\textit{Sequential pipeline}} \\
		\hspace{1em}Naive SVD & 		 35.29&	71.09&	57.85&	1.83&	8.73&	9.17&	37.72&	25.24&	30.87		\\
		\hspace{1em}+ Data whitening & 	 18.09&	87.12&	72.27&	29.27&	28.44&	48.60&	41.87&	46.22&	46.49		\\
		\hspace{1em}+ Bias calibration & 43.26&	87.04&	71.76&	33.54&	42.71&	61.41&	42.32&	46.47&	53.56		\\
		\bottomrule
	\end{tabular}
\end{table*}

\section{Baseline training configurations}
\label{sec:baseline-configs}
Both Deja Vu and Polar Sparsity predictors were trained with the pipelines and configs from the official repositories. Training configurations are provided in \cref{tab:training-configs}. 
Training took about 5 minutes per layer, i.e. about 3 hours per model. 
{
\setlength{\tabcolsep}{3pt}
	\begin{table}
		\caption{Training configurations for Deja Vu and Polar Sparsity predictors.}
		\centering
		{
			\footnotesize
			\begin{tabular}{lll}
				\toprule
				Hyperparameter & Deja Vu & Polar Sparsity \\
				\midrule
				Dataset          & C4       & WikiText \\
				Max Samples      & 500      & 5000 \\
				Tokens           & ~240k    & ~400k \\
				Train/Test       & 95/5     & 95/5 \\
				Batch Size       & 1024     & 64 \\
				Epochs           & 20     & 20 \\
				Loss             & BCE      & BCE \\
				Learning Rate    & 1e-4    & 1e-4 \\
				Architecture     & Linear($r=1024$) + ReLU + Linear, no bias & Linear($r=1024$) + Linear, no bias \\
				Early Stopping   & patience 5 & patience 5 \\
				Optimizer        & AdamW    & AdamW \\
				Weight Decay     & 0.01     & 0.01 \\
				\bottomrule
			\end{tabular}
		}
		\label{tab:training-configs}
	\end{table}
}

\section{Extension to non-ReLU activations}
\label{sec:non-relu-based-llms}
The solution can in principle be applied to LLMs with non-ReLU activation functions, such as SiLU. However, the straightforward utilization may lead to substantial accuracy degradation on long generation tasks. The quality evaluation results for Qwen2-7B-Instruct model with SiLU activation function are presented in \cref{tab:silu-results}.
Potential modifications to the current solution may include transformations to the intermediate activations to mitigate errors that arise.

\begin{table}[tb]
	\footnotesize
    \caption{Qwen2-7B-Instruct accuracy results at 25\% average predicted sparsity.}
    \centering
    {
        \begin{tabular}{lcccccccccc}
            \toprule
            Method              & TriviaQA & ARC-C & ARC-E & HumanEval & MBPP & GSM8K & BBH & CMMLU & Avg. \\
                               & acc, \% & acc, \% & acc, \% & pass@1, \% & pass@1, \% & acc, \% & acc, \% & acc, \% & Score \\
            \midrule
            Dense         & 60.96   & 86.43 & 93.81 & 60.98     & 44.46     & 36.85 & 47.81 & 79.58 & 63.86 \\
            \midrule
            SVDP     & 51.48   & 86.43 & 93.81 & 49.39     & 36.55     & 37.83 & 47.61 & 79.58 & 60.34 \\
            \bottomrule
        \end{tabular}
    }
    \label{tab:silu-results}
\end{table}

\section{Dynamic sparsity for batched processing}
\label{sec:batched-sparsity}
A power-law distribution exists in neuron activation within FFNs of LLMs, suggesting that the number of activated FFN neurons does not grow linearly with the batch size \cite{liu2023deja,song2024powerinferfastlargelanguage}. 
Nevertheless, \citet{shrestha2025polar} shows that activation sparsity in the FFNs degrades with batching. In line with this, many contextual sparsity methods, such as TEAL \cite{liu2024trainingfreeactivationsparsitylarge}, do not optimize the prefill stage.
As FFNs become compute-bound under batching, MoE-based sparsity approach outperforms neuron-based sparsity approach \cite{pmlr-v235-zhang24bq}.
In this scenario, training-free contextual sparsity predictors, such as SVDP, can be utilized as a replacement for training-free routers.
For instance, \citet{pmlr-v235-zhang24bq} and \citet{madhyastha202624sparsitybeyondneuronlevel} use k-means clustering to MoE-fy the dense model, accelerating the LLM pre-training through sparsity. 
While the learnable predictors are not employed because of high training costs, SVDP can be adapted for more efficient sparsity utilization.

\section{Sensitivity analysis to calibration data}
\label{sec:sensitivity-analysis}
\paragraph{Data size}
We analyzed the sensitivity of the predictor construction algorithm to the size of the calibration dataset (\cref{fig:sensitivity-analysis}). Since both data whitening and bias calibration require calibration tokens, we analyzed them independently. 
For data whitening (\cref{fig:analysis-svd}), we measure ROC AUC scores on held-out test sequences. 
For bias calibration, we measure FFN relative errors. 
As few as 10,000 tokens are enough for reliable construction, however small gains can be obtained by increasing the data size. This contrasts with the trainable approach, where hundreds of thousands of token samples are used. For instance, \cite{song2024prosparse} and \cite{shrestha2025polar} use 400,000 tokens.
Importantly, when performing SVD compression with a small number of tokens, $X X^{\top}$ is likely to be non-positive-definite, making Cholesky decomposition impossible. Instead, SVD-LLM V2 \cite{wang2025svdllm_v2} can be used.
Here, we also rely on weighted low-rank approximation \textit{with regularization} \cite{parkina2025coala}, which is effective when data is scarce:
\begin{equation}
\min_{\text{rank}(AB) \leq r}\|WX-ABX\|_F^2 + \mu\|W-AB\|_F^2,
\end{equation}
and we set $\mu=1$.
\begin{figure}[htbp]
    \centering
    \begin{subfigure}[b]{0.495\textwidth}
        \centering
        \includegraphics[width=\textwidth]{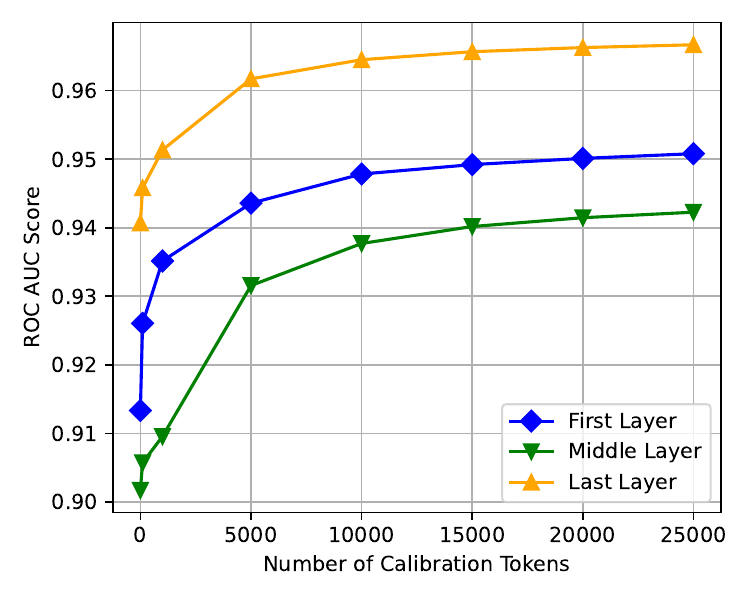} 
        \caption{Sensitivity of data whitening}
        \label{fig:analysis-svd}
    \end{subfigure}
    \hfill
    \begin{subfigure}[b]{0.495\textwidth}
        \centering
        \includegraphics[width=\textwidth]{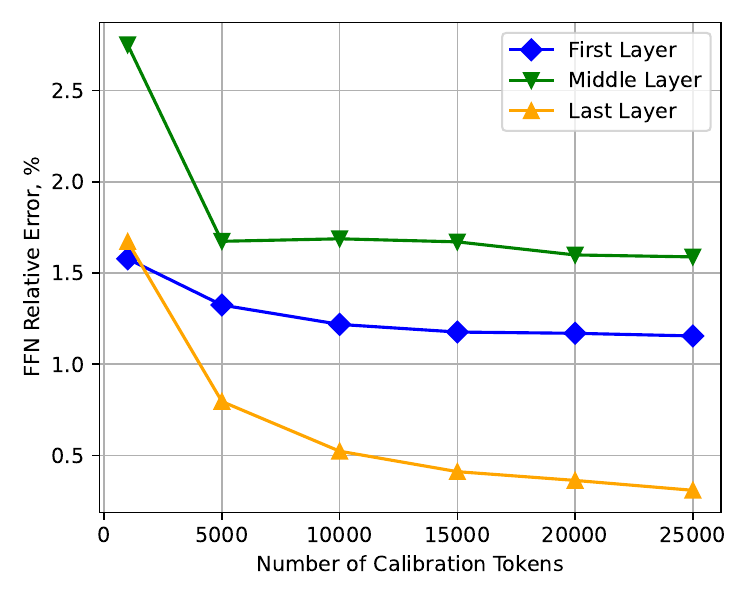} 
        \caption{Sensitivity of bias calibration}
        \label{fig:analysis-bias}
    \end{subfigure}
    \caption{Sensitivity analysis to calibration data size for first, middle and last layers of ProSparse-LLaMA2-7B. As few as 10,000 tokens are enough for reliable construction. Here, $r=256$. 
	 (a) For data whitening, sensitivity is measured with ROC AUC scores. Here, zero calibration tokens corresponds to naive SVD without data whitening.
	 (b) For bias calibration, sensitivity is measured with FFN relative error, averaged across test tokens. Here, $s=0.5$ and we keep the number of calibration tokens for data whitening fixed.}
    \label{fig:sensitivity-analysis}
\end{figure}

\paragraph{Data distribution}
To quantify what happens when calibration data is not drawn from the target application, we calibrate the whitening and bias of ProSparse-LLaMA2-7B on the dev split of CMMLU. 
We then evaluated on the remaining 7 benchmarks, all of which are English code, math, QA and reasoning tasks. 
CMMLU itself is excluded, since it is the calibration domain here. A degradation is present, though small (\cref{tab:data-distribution-ablation}).
\begin{table*}[h]
	\footnotesize
	\caption{
		Sensitivity of the algorithm to calibration data distribution for ProSparse-LLaMA2-7B ($r=256$, $s=0.5$).
		Here, $\checkmark$ — the stage was calibrated on data close to the downstream tasks (a mix of train splits); 
		$\times$ — the stage was calibrated on out-of-distribution data (CMMLU dev, Chinese).
	}
	\centering
	\begin{tabular}{cccccccccc}
		\toprule
		Whitening & Bias & TriviaQA & ARC-E & ARC-C & HumanEval & MBPP & GSM8K & BBH & Average \\
		\midrule
		$\times$ & $\times$ & 61.76 & 75.21 & 53.67 & 15.85 & 21.05 & 10.61 & 35.46 & 39.09 \\
		$\checkmark$ & $\times$ & 62.90 & 75.21 & 53.50 & 17.07 & 21.87 & 10.54 & 35.53 & 39.52 \\
		$\checkmark$ & $\checkmark$ & 62.74 & 75.34 & 53.67 & 17.68 & 22.07 & 11.22 & 35.58 & 39.76 \\
		\bottomrule
	\end{tabular}
	\label{tab:data-distribution-ablation}
\end{table*}
In practice it could be avoided, since calibration needs only $\approx$20,000 unlabeled tokens and about 10 minutes per 7B model.

\section{Comparison with standard threshold selection baselines}
\label{sec:alternative-biases}
To justify the necessity of the proposed greedy bias assignment we provide the ablation against uniform scaling and quantile matching.
\cref{alg:bias-calibration} dominates both on the Pareto front. At matched quality it is faster, and at matched speedup it is more accurate (\cref{fig:alternative-biases}).
\begin{figure}[htbp]
	\centering
	\includegraphics[width=0.6\textwidth]{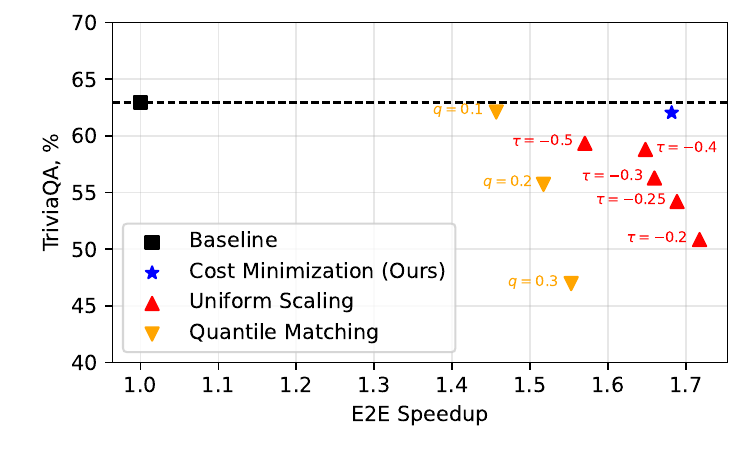}
	\caption{
		The ablation study against uniform scaling and quantile matching on TriviaQA benchmark for ProSparse-LLaMA2-7B.
		Uniform scaling applies one $\tau$ to every neuron, quantile matching sets each neuron's threshold at the $q$-quantile of its own calibration values, and Baseline is dense inference without prediction.
		Cost minimization is our \cref{alg:bias-calibration}. 
		}
	\label{fig:alternative-biases}
\end{figure}

\section{Technical details}
\begin{algorithm}
	\caption{Building SVD-predictors (Offline Phase)}
	\label{alg:svd-data-whitening}
	\begin{algorithmic}
  \STATE {\bfseries Input:}  $W^\mathrm{gate}\in\mathbb{R}^{D\times d}$, $X\in\mathbb{R}^{d\times N}$, $r\in\mathbb{N}$
	\STATE {\bfseries Output:} $A\in\mathbb{R}^{D\times r},\,B\in\mathbb{R}^{r\times d}$
  \STATE $S \in\mathbb{R}^{d\times d}\leftarrow\text{Cholesky}\left(X X^{\top}\right)$ 
  \STATE $U,\,\Sigma,\, V \leftarrow \text{SVD}\left(W^\mathrm{gate} S\right)$ 
  \STATE $A \in\mathbb{R}^{D\times r}\leftarrow U_{:,:r} \Sigma_{:r,:r}$ 
  \STATE $B\in\mathbb{R}^{r\times d} \leftarrow V_{:r,:}S^{-1}$ 
  \STATE {\bfseries return} $A,\,B$
	\end{algorithmic}
\end{algorithm}

\paragraph{LLMs hyperparameter values}
We provide hyperparameter values for the tested models in \cref{tab:model-hyperparams}, as they influence E2E latency. Specifically, all evaluated optimization techniques modify only the FFN portion of inference, leaving the Attention block unchanged. Thus, a higher intermediate size increases the FFN's contribution to E2E latency. Similarly, the use of Grouped Query Attention reduces the attention contribution to total inference time.
\begin{table*}[h]
	\footnotesize
	\caption{LLMs hyperparameter values.}
	\centering
	\begin{tabular}{lllll}
		\toprule
		& \multicolumn{4}{c}{hyperparameters}                   \\
		\cmidrule(r){2-5}
		Model     & Hidden size     & Intermediate size & N attention heads & N KV heads \\
		\midrule
		ProSparse-LLaMA2-7B &  4096 &  11008  & 32 & 32 \\
		TurboSparse-Mistral-Instruct & 4096 & 14336 & 32 & 8 \\
		SparseQwen2-7B     & 3584       & 18944  & 28 & 4 \\
		\bottomrule
	\end{tabular}
	\label{tab:model-hyperparams}
\end{table*}

\paragraph{Calibration data}
\label{sec:bias-calibration}
For calibration in total we use 112 few-shot samples ($\approx$20,000 tokens) from the train splits of the GSM8K, ARC-E, ARC-C, OpenBookQA, QAsper, CodeXGLUE datasets. We split these tokens in half for independent calibration of data whitening and bias.
Accuracies are measured on the benchmark test splits, the ROC-AUC curves on separately held-out sequences.

\paragraph{Predictor's rank}
\label{sec:rank-size}
One can trade off the predicted sparsity ratio and predictor size. 
While a larger predictor allows for more sparse computations within gate projection, it requires more compute for prediction. 
In turn, the $s$ hyperparameter can be used to calibrate the predicted sparsity ratio. 
\cref{fig:roc-auc-on-ranks} shows how the separation capability of SVD predictors depends on their size.
\begin{figure}
	\centering
	\includegraphics[scale =0.7]{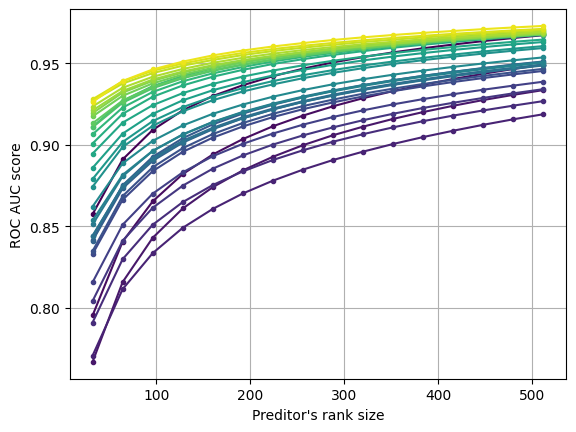}
	\caption{ROC AUC scores of SVD-predictors at different layers averaged across several sequence samples for ProSparse-LLaMA2-7B. The brighter curves represent deeper layers. For most of the layers a quality metric reaches a plateau at relatively small rank.}
	\label{fig:roc-auc-on-ranks}
\end{figure}
Due to this insight, higher acceleration can be achieved with relatively small predictors. Moreover, small predictors consume less additional memory. Similarly to $r$, we set $s$ to be constant across all layers. Empirically, we found that setting the rank to approximately 2\% of the intermediate size $D$ works well.

\section{Translating sparsity into end-to-end speedup}
\label{sec:detailed-profiling}
The executors do use the up-side zeros: the down-projection kernel selects active columns by the actual non-zeros of its input, not by the predicted gate mask (this applies to all baselines equally).
The sparsity decomposition by source for the dReLU models is provided in \cref{tab:sparsity-breakdown}. 
Overall, sparsity grows from the predictor output through the gate ReLU to the down projection, with the up-side ReLU adding the final 5–6\%. The sparsity varies considerably across layers and differs between models.
\Cref{tab:npu-res} benchmarks a single FFN block taken from the model, on synthetic inputs with the same sparsity at every stage. 
In real decoding, sparsity is non-uniform and rarely reaches 95\% on all layers, so the $30\times$ point is not reached on the full decode path. 
\begin{table*}[t]
	\setlength{\tabcolsep}{2pt}
	\footnotesize
	\centering
	\caption{ 
		Sparsity decomposition by source for the dReLU models on HumanEval ($\approx$20,000 generated tokens, $\text{mean} \pm \text{std}$).
		$^\dagger$Averaged across all layers.
	}
	\label{tab:sparsity-breakdown}
	\begin{tabular}{p{2.5cm}cccccccc}
		\toprule
		\multirow{2}{*}{\makecell[c]{Sparsity Source}}
		& \multicolumn{4}{c}{TurboSparse-Mistral-Instruct}
		& \multicolumn{4}{c}{SparseQwen2-7B} \\
		\cmidrule(lr){2-5} \cmidrule(l){6-9}
		& First Layer & Middle Layer & Last Layer & $\text{Avg.}^\dagger$
		& First Layer & Middle Layer & Last Layer & $\text{Avg.}^\dagger$ \\
		\midrule
		gate-predicted        & $50 \pm 5$ & $51 \pm 4$ & $26 \pm 6$ & $55 \pm 4$
		                      & $35 \pm 7$ & $43 \pm 4$ & $59 \pm 11$ & $59 \pm 3$ \\
		gate-side-correction  & $86 \pm 4$ & $80 \pm 2$ & $83 \pm 4$ & $85 \pm 1$
		                      & $79 \pm 3$ & $70 \pm 2$ & $90 \pm 2$ & $83 \pm 1$ \\
		up-side               & $74 \pm 4$ & $55 \pm 2$ & $53 \pm 2$ & $60 \pm 1$
		                      & $48 \pm 1$ & $35 \pm 2$ & $37 \pm 2$ & $37 \pm 1$ \\
		gate-up-union         & $94 \pm 2$ & $87 \pm 2$ & $90 \pm 2$ & $91 \pm 1$
		                      & $88 \pm 1$ & $78 \pm 2$ & $93 \pm 2$ & $88 \pm 1$ \\
		\bottomrule
	\end{tabular}
\end{table*}
\cref{fig:detailed-profiling} presents the decode latency breakdown during inference for both the dense baseline and the SVDP-accelerated version.
The E2E gain is bounded by the FFN share of the decode step.
\begin{figure}
    \centering
    \includegraphics[width=0.5\textwidth]{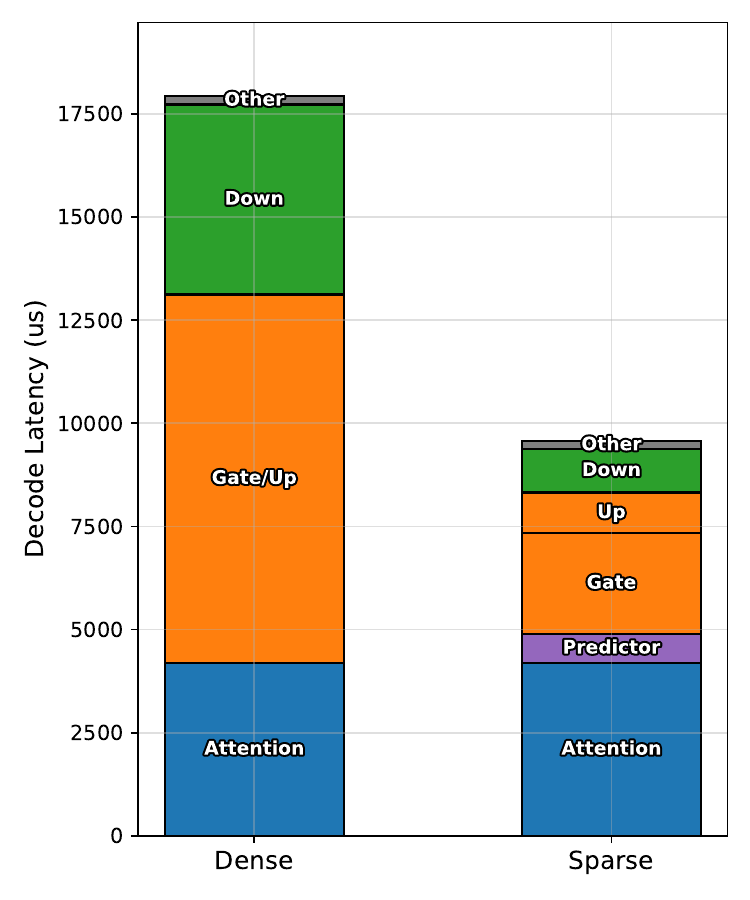}
    \caption{Decode latency breakdown obtained with the PyTorch profiler on the first decoded token of a sample sequence for TurboSparse-Mistral-Instruct. The ``Other’’ block comprises normalizations, logits processing, and sampling.
	FFNs dominate end-to-end latency during single-batch decoding on the GPU. 
	Even without asynchronous prediction, SVDP combined with the sequential pipeline achieves over $2.5\times$ FFN speedup. 
	Note that the speedup depends on the token-dependent sparsity ratio.
	}
    \label{fig:detailed-profiling}
\end{figure}

\section{Licenses for existing assets}
\label{sec:licenses}
The licences and versions for used assets are provided in \cref{tab:assets}.
{
	\begin{table}[h]
	\caption{Assets, licences, versions, and URLs used in this work.}                                                             
	\centering
	\footnotesize
	\label{tab:assets}
	\begin{tabular}{llll}
	\toprule
	Asset & Licence & Version & URL \\
	\midrule
	UltraEval &  Apache-2.0 & from commit 5d967b4 & \url{github.com/OpenBMB/UltraEval} \\
	vLLM & Apache-2.0 & 0.8.5.post1 & \url{github.com/vllm-project/vllm}\\
	GSM8K & MIT & accessed 2026-01-15 & \url{huggingface.co/datasets/openai/gsm8k}\\
	HumanEval & MIT & accessed 2026-01-15 & \url{huggingface.co/datasets/openai/openai_humaneval}\\
	ARC-E/ARC-C & CC BY-SA 4.0 & accessed 2026-01-15 & \url{huggingface.co/datasets/allenai/ai2_arc}\\
	TriviaQA & Apache-2.0 & accessed 2026-01-15 & \url{huggingface.co/datasets/mandarjoshi/trivia_qa}\\
	MBPP & CC BY 4.0 & accessed 2026-01-15 & \url{huggingface.co/datasets/Muennighoff/mbpp}\\
	BBH & MIT & accessed 2026-01-15 & \url{huggingface.co/datasets/lukaemon/bbh}\\
	CMMLU & CC BY-NC 4.0 & accessed 2026-01-15 & \url{huggingface.co/datasets/lmlmcat/cmmlu}\\
	\bottomrule
	\end{tabular}
	\end{table}
}


\newpage
\section*{NeurIPS Paper Checklist}

\begin{enumerate}

\item {\bf Claims}
    \item[] Question: Do the main claims made in the abstract and introduction accurately reflect the paper's contributions and scope?
    \item[] Answer: \answerYes{} 
    \item[] Justification: We describe our scope and contributions in the abstract and \cref{sec:introduction}.

\item {\bf Limitations}
    \item[] Question: Does the paper discuss the limitations of the work performed by the authors?
    \item[] Answer: \answerYes{} 
    \item[] Justification: We discuss the limitations of our method in \cref{sec:limitations}.

\item {\bf Theory assumptions and proofs}
    \item[] Question: For each theoretical result, does the paper provide the full set of assumptions and a complete (and correct) proof?
    \item[] Answer: \answerYes{} 
    \item[] Justification: The theoretical results and a complete proofs are provided in \cref{sec:math-proofs}

    \item {\bf Experimental result reproducibility}
    \item[] Question: Does the paper fully disclose all the information needed to reproduce the main experimental results of the paper to the extent that it affects the main claims and/or conclusions of the paper (regardless of whether the code and data are provided or not)?
    \item[] Answer: \answerYes{} 
    \item[] Justification: We describe the pipeline of our method in \cref{sec:svd-predictors} and experiment details are presented in \cref{sec:experimental-results}. We provide the complete codes as supplementary materials.

\item {\bf Open access to data and code}
    \item[] Question: Does the paper provide open access to the data and code, with sufficient instructions to faithfully reproduce the main experimental results, as described in supplemental material?
    \item[] Answer: \answerYes{} 
    \item[] Justification: The complete codes are included.

\item {\bf Experimental setting/details}
    \item[] Question: Does the paper specify all the training and test details (e.g., data splits, hyperparameters, how they were chosen, type of optimizer) necessary to understand the results?
    \item[] Answer: \answerYes{} 
    \item[] Justification: Experiment details are given in \cref{sec:experimental-results}, \cref{sec:detailed-acc-results}, \cref{sec:bias-calibration}. The complete codes are included.

\item {\bf Experiment statistical significance}
    \item[] Question: Does the paper report error bars suitably and correctly defined or other appropriate information about the statistical significance of the experiments?
    \item[] Answer: \answerYes{} 
    \item[] Justification: We report standard deviation of speedup values to capture low run-to-run variability across independent runs given experimental conditions.
     To ensure robustness, we evaluate our method across a range of LLM architectures, and downstream tasks, as detailed in \cref{sec:experimental-results}.


\item {\bf Experiments compute resources}
    \item[] Question: For each experiment, does the paper provide sufficient information on the computer resources (type of compute workers, memory, time of execution) needed to reproduce the experiments?
    \item[] Answer: \answerYes{} 
    \item[] Justification: The information of the employed compute resources is described in \cref{sec:experimental-results}.
    
\item {\bf Code of ethics}
    \item[] Question: Does the research conducted in the paper conform, in every respect, with the NeurIPS Code of Ethics \url{https://neurips.cc/public/EthicsGuidelines}?
    \item[] Answer: \answerYes{} 
    \item[] Justification: The research conducted in the paper conform, in every respect, with the NeurIPS Code of Ethics.

\item {\bf Broader impacts}
    \item[] Question: Does the paper discuss both potential positive societal impacts and negative societal impacts of the work performed?
    \item[] Answer: \answerYes{} 
    \item[] Justification: We discuss potential positive societal impacts of the work in \cref{sec:impact}.
    
\item {\bf Safeguards}
    \item[] Question: Does the paper describe safeguards that have been put in place for responsible release of data or models that have a high risk for misuse (e.g., pre-trained language models, image generators, or scraped datasets)?
    \item[] Answer: \answerNA{} 
    \item[] Justification: The released code has no risks of misuse.

\item {\bf Licenses for existing assets}
    \item[] Question: Are the creators or original owners of assets (e.g., code, data, models), used in the paper, properly credited and are the license and terms of use explicitly mentioned and properly respected?
    \item[] Answer: \answerYes{} 
    \item[] Justification: We correctly cite all assets and adhere to the usage restrictions. Licenses and versions for used assets are listed in \cref{sec:licenses}.

\item {\bf New assets}
    \item[] Question: Are new assets introduced in the paper well documented and is the documentation provided alongside the assets?
    \item[] Answer: \answerYes{} 
    \item[] Justification: The code and other supplementary materials are followed with readme and instructions.

\item {\bf Crowdsourcing and research with human subjects}
    \item[] Question: For crowdsourcing experiments and research with human subjects, does the paper include the full text of instructions given to participants and screenshots, if applicable, as well as details about compensation (if any)? 
    \item[] Answer: \answerNA{} 
    \item[] Justification: We do not involve any crowdsourcing or research with human subjects.

\item {\bf Institutional review board (IRB) approvals or equivalent for research with human subjects}
    \item[] Question: Does the paper describe potential risks incurred by study participants, whether such risks were disclosed to the subjects, and whether Institutional Review Board (IRB) approvals (or an equivalent approval/review based on the requirements of your country or institution) were obtained?
    \item[] Answer: \answerNA{} 
    \item[] Justification: We do not involve any crowdsourcing or research with human subjects.

\item {\bf Declaration of LLM usage}
    \item[] Question: Does the paper describe the usage of LLMs if it is an important, original, or non-standard component of the core methods in this research? Note that if the LLM is used only for writing, editing, or formatting purposes and does \emph{not} impact the core methodology, scientific rigor, or originality of the research, declaration is not required.
    \item[] Answer: \answerNA{} 
    \item[] Justification: The core method development in this research does not involve LLMs as any important, original, or non-standard components.

\end{enumerate}

\end{document}